\documentclass[journal]{IEEEtran}

\usepackage{cite}
\usepackage[cmex10]{amsmath}
\usepackage{algorithmic}
\usepackage{array}
\usepackage{booktabs}
\usepackage{multirow}
\usepackage{url}
\usepackage{footnote}
\usepackage{float}
\usepackage[utf8]{inputenc}

\usepackage[T1]{fontenc}
\usepackage[ruled,vlined]{algorithm2e}
\usepackage{ifthen}
\usepackage{graphicx}
\usepackage{comment}
\usepackage{amsthm}
\usepackage{amsmath}
\usepackage{amsmath,amssymb,amsfonts}
\usepackage{nicefrac}
\usepackage{algorithmic}
\usepackage{bm}
\usepackage{mathtools}
\usepackage{tikz}
\usetikzlibrary{calc,math,arrows.meta,positioning,fit,backgrounds}
\usepackage{standalone}
\usepackage[framemethod=TikZ]{mdframed}
\usepackage{pgfplots}
\pgfplotsset{compat=1.16}
\usepackage{hyperref}
\hypersetup{colorlinks=true, unicode=true, linkcolor=[rgb]{0.10,0.05,0.67}, citecolor=[rgb]{0.10,0.05,0.67}, filecolor=[rgb]{0.10,0.05,0.67}, urlcolor=[rgb]{0.10,0.05,0.67}}

\usepackage{xcolor}

\newtheorem{remark}{Remark}
\newtheorem{definition}{Definition}

\newtheorem{prop}{Proposition}

\DeclareMathOperator*{\argmin}{arg\,min}

\makeatletter
\renewcommand*{\eqref}[1]{%
  \hyperref[{#1}]{\textup{\tagform@{\ref*{#1}}}}%
}
\makeatother

\begin{document}

\title{{UniTAC}: Universal Task-Aware Compression via Weighted
Distortion Measures}
\author{%
    {Homa~Esfahanizadeh,~\IEEEmembership{Member,~IEEE,}
    Matin~Mortaheb,~\IEEEmembership{Member,~IEEE,}
    Adeel~Mahmood,~\IEEEmembership{Member,~IEEE,}
    Jinfeng~Du,~\IEEEmembership{Senior~Member,~IEEE,}
    and~Harish~Viswanathan,~\IEEEmembership{Fellow,~IEEE}}%
    \thanks{The authors are with Nokia Bell Labs, Murray Hill, NJ 07974, USA
    (e-mail: homa.esfahanizadeh@nokia-bell-labs.com; matin.mortaheb@nokia-bell-labs.com;
    adeel.mahmood@nokia.com; jinfeng.du@nokia-bell-labs.com;
    harish.viswanathan@nokia-bell-labs.com). \emph{(Corresponding author: Homa Esfahanizadeh.)}}%
}

\maketitle

\begin{abstract}
Lossy compression is conventionally driven by a task-agnostic distortion (e.g., MSE or MS-SSIM), yet in many emerging applications the receiver cares not about uniform fidelity but about a downstream task whose relevant content varies across the signal and evolves over time. We formulate task-aware compression as a weighted rate--distortion problem, in which a single codec is driven by a separable, per-component \emph{weighted distortion} whose weights encode task importance and may depend on the source. We introduce \emph{task consistency}, i.e., that minimizing the weighted distortion also minimizes the true task loss, and characterize when it holds: for linear tasks, the task loss admits a weighted-MSE form with signal-independent weights under suitable cross-term conditions, while for nonlinear tasks, an integrated-gradients analysis motivates separable task-aware weights.
We show how task symmetry and irrelevance further constrain the admissible weights. Guided by this theory, we realize the weight-conditioned code in a single learned Vision Transformer (ViT) codec whose token-level attention natively consumes a per-component importance vector, so one fixed backbone is re-targeted at runtime, from universal (task-agnostic) to task-specialized operation, purely by swapping the injected weights, without retraining, while producing a single human-viewable reconstruction steered to the active task. On downstream face-analysis tasks, a single model reaches 91.4\% accuracy at 0.034 bpp on a localized task, within 1.9\% of a task-specific codec (93.3\%) and well above a universal codec (76.9\%). Such task-adaptive compression suits bandwidth-constrained perception systems, e.g., in Physical AI, where the active task drifts and per-task retraining is infeasible.
\end{abstract}


\section{Introduction}

Lossy compression is fundamentally a trade-off between rate and distortion, where fewer bits mean a less faithful reconstruction. In many emerging applications, however, the reconstruction is ultimately consumed by a downstream task, for which different parts of the signal carry unequal importance. Classical rate-distortion theory already anticipates this by allowing distortion to be measured differently across different parts of the signal~\cite{shannon1959coding}.
Yet generic learned image and feature codecs, however, are task-agnostic. They optimize rate-distortion objectives with respect to PSNR or MS-SSIM, and thus spend bits uniformly regardless of what the receiver needs~\cite{balle2018variational,esenlik2025jpegai}.

One such application domain is Physical AI, i.e., AI systems that perceive and act in the physical world, typically closing feedback loops between sensing, communication, and control~\cite{duan2022embodied}. In such systems, raw sensory streams are high-rate and redundant, yet the downstream task rarely needs all bits equally. For example, a mobile robot may only need an accurate reconstruction of regions associated with obstacles, grasp points, or safety-critical agents. On the other hand, Physical AI systems often communicate over tightly constrained links that are bandwidth-limited, time-varying, and noisy, so every bit spent on task-irrelevant content is a bit unavailable for task-relevant content or for protecting it against channel impairments. Task-based compression is therefore most valuable precisely under such harsh channel conditions. In this paper, we isolate the source-coding side of this problem, i.e., how to allocate a fixed, and possibly scarce, rate budget across the task-relevant content of the source.

Building task awareness into compression faces two practical obstacles:
\begin{itemize}
    \item \textbf{Task evolution:} In many applications, the active task changes with context. For instance, a mobile robot's task can shift from navigation to mapping to interaction. Designing and retraining a bespoke codec for each task, e.g.,~\cite{singh2020compressible,le2021icm}, is operationally expensive and often infeasible in the field.
    \item \textbf{Systematic definition and integration of the region of interest}: Prior methods rely on heuristics, such as text prompts or application metadata, to indicate salient regions~\cite{feng2023prompticm,iso15444}, but a principled mapping from the task to per-component importance, together with a way to integrate it systematically into the compressor, is lacking.
\end{itemize}
This paper develops a unifying approach that combines the benefits of universal (task-agnostic) and semantic (task-aware) compression; see Fig.~\ref{fig:paradigms}.

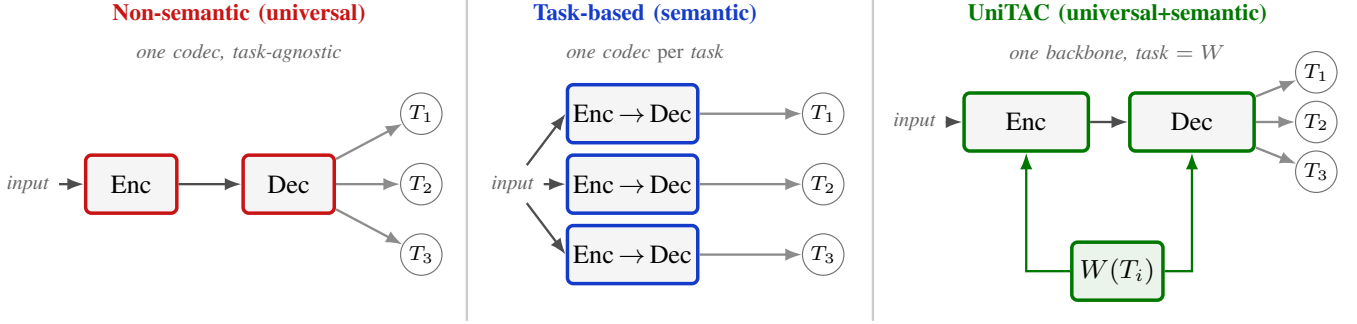
\begin{figure*}[t]
    \centering
    \resizebox{\textwidth}{!}{\definecolor{codegreen}{RGB}{0,120,0}
\definecolor{codered}{RGB}{200,20,20}
\definecolor{codeblue}{RGB}{20,60,200}

\begin{tikzpicture}[
    font=\small,
    >={Latex[length=2.0mm]},
    box/.style={draw, rounded corners=2pt, minimum height=7mm, minimum width=11mm,
                align=center, inner sep=2pt},
    codec/.style={box, fill=black!4},
    task/.style={draw, circle, minimum size=5mm, inner sep=0pt, font=\scriptsize},
    lab/.style={font=\footnotesize\bfseries},
    sub/.style={font=\scriptsize\itshape, text=black!60},
    flow/.style={->, thick, black!70},
]

\def\colw{4.6}        
\def\ax{0}            
\def\bx{4.9}          
\def\cx{10.6}         
\def\rowtop{0}

\begin{scope}[shift={(\ax,0)}]
  \node[lab, text=codered] (atitle) at (2.0,1.75) {Non-semantic (universal)};
  \node[sub] at (2.0,1.28) {one codec, task-agnostic};

  \node[codec, draw=codered, very thick] (aenc) at (0.7,-0.3) {Enc};
  \node[codec, draw=codered, very thick] (adec) at (2.6,-0.3) {Dec};
  \draw[flow] (aenc) -- (adec);

  \node[left=3mm of aenc, sub, align=center] (ain) {input};
  \draw[flow] (ain) -- (aenc);

  \node[task, draw=black!55] (at1) at (4.2, 0.55) {$T_1$};
  \node[task, draw=black!55] (at2) at (4.2,-0.30) {$T_2$};
  \node[task, draw=black!55] (at3) at (4.2,-1.15) {$T_3$};
  \draw[flow, black!45] (adec) -- (at1);
  \draw[flow, black!45] (adec) -- (at2);
  \draw[flow, black!45] (adec) -- (at3);
\end{scope}

\begin{scope}[shift={(\bx,0)}]
  \node[lab, text=codeblue] (btitle) at (2.0,1.75) {Task-based (semantic)};
  \node[sub] at (2.0,1.28) {one codec \emph{per} task};

  \node[codec, draw=codeblue, very thick] (bc1) at (1.85, 0.55) {Enc\,$\to$\,Dec};
  \node[codec, draw=codeblue, very thick] (bc2) at (1.85,-0.30) {Enc\,$\to$\,Dec};
  \node[codec, draw=codeblue, very thick] (bc3) at (1.85,-1.15) {Enc\,$\to$\,Dec};

  \node[sub, align=center] (bin) at (0.4,-0.30) {input};
  \draw[flow] (bin) -- (bc1.west);
  \draw[flow] (bin) -- (bc2.west);
  \draw[flow] (bin) -- (bc3.west);

  \node[task, draw=black!55] (bt1) at (4.15, 0.55) {$T_1$};
  \node[task, draw=black!55] (bt2) at (4.15,-0.30) {$T_2$};
  \node[task, draw=black!55] (bt3) at (4.15,-1.15) {$T_3$};
  \draw[flow, black!45] (bc1) -- (bt1);
  \draw[flow, black!45] (bc2) -- (bt2);
  \draw[flow, black!45] (bc3) -- (bt3);
\end{scope}

\begin{scope}[shift={(\cx,0)}]
  \node[lab, text=codegreen] (ctitle) at (2.0,1.75) {UniTAC (universal+semantic)};
  \node[sub] at (2.0,1.28) {one backbone, task $=W$};

  \node[codec, draw=codegreen, very thick, minimum width=15mm] (cenc) at (0.9,0.45) {Enc};
  \node[codec, draw=codegreen, very thick, minimum width=15mm] (cdec) at (2.9,0.45) {Dec};
  \draw[flow] (cenc) -- (cdec);

  \node[left=2mm of cenc, sub, align=center] (cin) {input};
  \draw[flow] (cin) -- (cenc);

  \node[task, draw=black!55] (ct1) at (4.4, 1.05) {$T_1$};
  \node[task, draw=black!55] (ct2) at (4.4, 0.45) {$T_2$};
  \node[task, draw=black!55] (ct3) at (4.4,-0.15) {$T_3$};
  \draw[flow, black!45] (cdec) -- (ct1);
  \draw[flow, black!45] (cdec) -- (ct2);
  \draw[flow, black!45] (cdec) -- (ct3);

  \node[box, draw=codegreen, fill=codegreen!8, very thick] (cw) at (2.0,-1.35)
       {$W(T_i)$};
  \draw[->, thick, codegreen] (cw.west) -| (cenc.south);
  \draw[->, thick, codegreen] (cw.east) -| (cdec.south);
\end{scope}

\draw[black!20, thick] (4.75,-1.95) -- (4.75,1.95);
\draw[black!20, thick] (9.65,-1.95) -- (9.65,1.95);

\end{tikzpicture}}\vspace{-0.3cm}
    \caption{Three regimes of learned compression. \textbf{(left) Non-semantic (universal):} a single task-agnostic codec spends bits uniformly, so it serves every task but is specialized to none. \textbf{(middle) Task-based (semantic):} a separate codec is trained for each task, so task drift forces retraining a new codec. \textbf{(right) UniTAC (ours):} a single shared backbone in which the task is abstracted as a per-component importance vector $W$ injected at runtime; swapping $W$ specializes the same model to any task without retraining.\vspace{-0.3cm}}
    \label{fig:paradigms}
\end{figure*}

We represent the signal as a collection of components $\{\boldsymbol{x}_i\}_{i=1}^{n}$ and drive a single compressor with a \textit{weighted distortion} $\sum_{i}w_i D(\boldsymbol{x}_i,\boldsymbol{\hat{x}}_i)$, where $\boldsymbol{\hat{x}}_i$ is the compressor's reconstruction of component $\boldsymbol{x}_i$ and the weights $\{w_i\}$ encode task saliency and can depend on the signal. Our contributions are threefold. \emph{(i) Framework:} we cast task-aware compression as weight-conditioned compression, in which a task descriptor is injected as a per-component importance vector $W$ that conditions both the encoder and decoder. The compressor is trained once over a broad family of such vectors, so that task drift is handled at runtime by updating $W$ rather than retraining. \emph{(ii) Theory:} we analyze the underlying weighted rate-distortion problem and characterize when a diagonal weighted distortion is \emph{task-consistent} and how the weights relate to task sensitivity, giving a principled \emph{task to importance} mapping. \emph{(iii) Design and evaluation:} guided by this analysis, we design a Vision Transformer (ViT) codec whose token-level conditioning natively realizes the weight-driven code, and show on downstream tasks that a single UniTAC model approaches a task-specific codec while improving over a universal one at nearly equal rate. 
Taken together, these contributions position UniTAC between the two extremes of Fig.~\ref{fig:paradigms}: unlike a task-agnostic codec, it targets the content that matters for the active task, and unlike a task-specific codec, it does so within a single trained backbone that requires no retraining as the task changes.

The remainder of the paper is organized as follows. Section~\ref{sec:related} reviews related work on learned compression and task-oriented communication. Section~\ref{sec:framework} introduces the weighted-distortion framework and the notion of task consistency. Section~\ref{sec:theory} develops the theoretical analysis linking the weights to task sensitivity. Section~\ref{sec:design} presents the ViT-based UniTAC codec, and Section~\ref{sec:eval} reports the experimental evaluation. Section~\ref{sec:structural} establishes further structural properties of task-consistent weights. Section~\ref{sec:conclusion} concludes.

\section{Related Work}\label{sec:related}

\subsection{Rate--distortion theory with task-dependent distortion}\label{sec:related-rd}

Our formulation is an instance of classical rate--distortion theory, which characterizes the optimal tradeoff between compression rate and a distortion measure between a source (here, vector-valued) and its reconstruction~\cite{shannon1959coding}. The original treatment permits general distortion measures, including context-dependent local distortion criteria. Because our task loss is evaluated on a downstream function of the source rather than on the source itself, our problem is also an instance of \emph{indirect} rate-distortion theory, in which fidelity is measured against a variable correlated with, but not identical to, the coded source~\cite{Witsenhausen_indirect_RD}. That classical theory characterizes the fundamental rate-distortion tradeoff for a \emph{fixed} indirect distortion criterion. We instead ask when a \emph{separable, per-component} surrogate distortion, whose weights may vary per sample and which is injected into a single learned codec, is consistent with that criterion, a question we term \emph{task consistency}.

\subsection{Universal image and feature compression}
Classical rate--distortion codecs (JPEG, BPG, HEVC/VVC intra) and learned image codecs based on nonlinear transform coding with factorized or hyperprior entropy models~\cite{balle2017end,balle2018variational} are all trained to minimize a generic distortion (PSNR / MS-SSIM) at a given rate. This effort has culminated in JPEG AI (ISO/IEC~6048 | ITU-T~T.840)~\cite{jpegai_standard,esenlik2025jpegai}, the first end-to-end learned image-coding standard, which specifies convolutional analysis/synthesis transforms, a hyperprior with a multistage context entropy model, and a 3D gain unit for spatially variable quantization and region-of-interest control. All of these codecs are inherently task-agnostic and allocate rate uniformly across the image. Even JPEG AI is trained on a fixed distortion (a weighted combination of MSE and MS-SSIM), and its gain unit provides only a manually specified region-of-interest rather than a task-derived importance signal that conditions the coding process. A related information-theoretic work establishes achievability for universal lossy compression when the distortion measure is specified only at runtime~\cite{mahmood2023universal}.

\subsection{Task-driven and information-bottleneck source coding}
A growing body of work departs from reconstruction-centric coding and instead optimizes transmission for a downstream objective, known as task-oriented or semantic communication~\cite{gunduz2023beyond,strinati2021semantic}. 
A representative approach~\cite{singh2020compressible} trains a feature extractor, task head, and entropy model end-to-end to minimize a task loss plus a rate penalty on the features. Such single-task feature codecs are information-theoretically driven toward the minimal task-sufficient statistic (a soft label at the optimum) and therefore discard the input. Other work such as~\cite{le2021icm,duan2020vcm} instead trains the codec against a combination of a task loss, a generic distortion loss, and a rate loss. The distortion term keeps these codecs reconstructable, but the task is still baked into the trained weights, so serving a new task requires retraining.

The information-bottleneck (IB) principle~\cite{tishby1999ib} characterizes the optimal trade-off between the rate of a compressed representation and the information it retains about a relevant target variable, and its deep variational realization~\cite{alemi2017vib} makes this trade-off trainable in neural networks. Practical IB-based methods shape data for a specific downstream task by optimizing mutual-information objectives, e.g., InfoShape for images~\cite{esfahanizadeh2023infoshape} and TexShape for sentence embeddings~\cite{kale2024texshape}. These formulations analyze \emph{what} a task-relevant representation should retain, but they optimize a single fixed relevance variable, yield a representation specialized to that one objective, and neither reconstruct the source nor re-specialize a single codec to a new task at runtime.

\subsection{Multi-task and adaptive codecs}
Closest to UniTAC is a line of multi-task codecs, which fall into two categories: coding for machines, and coding for both machines and humans. The latter is more aligned with Physical AI, where a decoded signal must serve a downstream task while remaining usable for human viewing.

Prompt-based image coding for machines~\cite{feng2023prompticm} conditions a single feature codec with task-driven ``prompts.'' Each prompt is an importance map that tells the codec to spend more bits on task-relevant regions and fewer elsewhere, giving spatially uneven, content-aware bit allocation. Because the codec is trained on a wide, randomized family of such maps, one model can serve several downstream tasks simply by being given a different map. It differs from UniTAC in several ways. It compresses and reconstructs backbone features under a downstream task loss and thus yields no human-viewable image. Its importance map is produced by a learned information selector that, together with task-adaptive prompts injected into the downstream network, must be fine-tuned per task, so a task-specific parameter set is still required at inference. Finally, its importance map conditions only the encoder, is never transmitted, and drives only local, position-wise modulation rather than long-range, content-dependent interaction. Multi-Path Aggregation~\cite{zhang2024mpa} serves both human viewing and machine vision from one transformer codec by inserting a shared main path plus per-task side paths. For this method, a dedicated side path (and predictor) is added and fine-tuned per task in a second training stage, supporting only a fixed, predefined task set rather than arbitrary or unseen tasks. Scalable human--machine coding~\cite{choi2022scalable} likewise serves both consumers, but through a fixed, predefined layered task hierarchy rather than a runtime importance vector.

Unlike these task/importance-conditioned codecs, UniTAC is agnostic to how the importance vector is derived. It is a single universal image codec that receives the importance vector as transmitted low-overhead side information, conditions both encoder and decoder on it, and reconstructs a human-viewable, general-purpose image whose fidelity is prioritized for the targeted task, enabling runtime re-targeting to any task without retraining.

A related line inserts task/importance-awareness at the \emph{channel}-coding level rather than the source-coding level: task-aware joint source--channel coding (JSCC), which maps source content directly to channel symbols rather than through separate compression and channel codes~\cite{bourtsoulatze2019deepjscc,tung2022deepwive}, folds semantic or task importance into \emph{unequal channel error protection}, e.g., via a multi-level reliability interface that maps source semantics to graded channel reliability~\cite{tung2025multilevel}, or error-resilient, low-latency video transmission that prioritizes semantically important content against block erasures and channel impairments~\cite{fayaz2025robust,mortaheb2026semantic}. This is complementary to UniTAC, which assumes an error-free transport layer and instead allocates rate at the source-coding stage.

\subsection{Task-to-importance mapping}
A body of explainability research identifies which parts of an input a trained model relies upon, producing per-input importance or saliency maps. These methods fall into a few families: gradient- and activation-based methods that back-propagate or pool network responses onto the input (e.g., Grad-CAM~\cite{selvaraju2017gradcam}); perturbation- and mask-based methods that find the smallest region whose deletion or retention most changes the prediction~\cite{fong2017perturbation}; and path-attribution methods that integrate gradients along a path from a baseline to the input. Integrated Gradients~\cite{sundararajan2017axiomatic} is the canonical path method, requiring a reference baseline and satisfying desirable axioms such as completeness. Its variants average over baselines or distributions (e.g., SmoothGrad~\cite{smilkov2017smoothgrad} and expected gradients / attribution priors~\cite{erion2019attribution}) for smoother, distribution-aware maps. Related sensitivity measures based on the task Jacobian or Fisher information similarly quantify how strongly each input component influences the output. This prior art was developed for interpretability, and does not connect the resulting importance to rate allocation, condition an encoder/decoder on it, or use it to weight a reconstruction distortion. UniTAC repurposes such task sensitivity as a principled task-to-importance signal.

\section{A Weighted-Distortion Framework for Task-Conditioned Compression}\label{sec:framework}

Throughout, boldface symbols denote random quantities and non-bold symbols their realizations (or deterministic quantities); uppercase letters denote vectors and matrices, and lowercase letters denote scalars and sub-vectors/sub-matrices (components, columns, and entries); and calligraphic letters denote sets. Thus, $\boldsymbol{X}$ is a random vector with realization $X$, $\boldsymbol{x}_i$ is a random component with realization $x_i$, $W$ is a (deterministic) weight vector with entries $w_i$, and $\mathcal{P},\mathcal{S}$ denote sets. We write $\mathbb{R}$ and $\mathbb{R}_{+}$ for the reals and nonnegative reals, $(\cdot)^T$ for transpose, and $\mathbb{E}[\cdot]$ for expectation. 

Let $\boldsymbol{X}$ and $\boldsymbol{\hat{X}}$ denote the source and reconstruction vectors, respectively. Mathematically, we assume that the reconstruction is produced via a \emph{test channel} $P(\hat{X}{\mid}X)$, which is a conditional distribution mapping a source realization to a reconstruction and abstracts any (possibly stochastic) encoder-decoder pair. Let the task be specified by a function $f:\mathbb{R}^n\rightarrow\mathbb{R}^m$, which may be applied either to the original source vector (yielding $f(\boldsymbol{X})$) or to its reconstruction (yielding $f(\boldsymbol{\hat{X}})$). A natural task-fidelity criterion is the mean squared error (MSE) in task space:
\begin{equation}
    \mathcal{L}_\text{task}(P)\triangleq \mathbb{E}\left[||f(\boldsymbol{X})-f(\boldsymbol{\hat{X}})||_2^2\right],
\end{equation}
where we write $P = P(\hat{X}{\mid}X)$ as shorthand, and the expectation above is with respect to the joint distribution $P_{\boldsymbol{X}} \times P$. Our key proposal is to compress against a separable weighted distortion $D_{W}(P)=\mathbb{E}[D_{W}(\boldsymbol{X},\boldsymbol{\hat{X}})]$, where
\begin{equation}\label{eq:separable_w_distortion}
    D_{W}({X},{\hat{X}})\triangleq\sum_{i=1}^n w_i({X})\, D({x}_i,{\hat{x}}_i),\;\;w_i({X})\geq0.
\end{equation}
Here, $D({x}_i,{\hat{x}}_i)\geq 0$ is a per-component distortion (e.g., the squared error $||{x}_i-{\hat{x}}_i||_2^2$) and the weight $w_i({X})$ quantifies the task importance of the $i$th component (i.e., how strongly changes in $x_i$ influence the task output $f({X})$). The weights may be \emph{sample-dependent} and produced by a weight map $W:\mathbb{R}^n\to\mathbb{R}_{+}^n$.

The compression rate is measured by the mutual information $I(\boldsymbol{X};\boldsymbol{\hat{X}})$, and the \emph{feasible set} at rate budget $R$ is the collection of all test channels satisfying this budget:
\begin{equation}\label{eq:feasible_set}
    \mathcal{P}(R)\triangleq\left\{\,P(\hat{X}{\mid}X):\, I(\boldsymbol{X};\boldsymbol{\hat{X}})\leq R\,\right\}.
\end{equation}
The goal is to choose the weights so that a rate-constrained compressor, abstracted as a test channel $P \in \mathcal{P}(R)$, minimizing the weighted distortion $D_{W}(P)$ also drives down the task distortion $\mathcal{L}_\text{task}(P)$; when it does, we call the weighted distortion \emph{task-consistent}. 
The results in this paper address when minimizing the \emph{weighted distortion} $D_{W}(P)$ (the separable surrogate optimized by the codec) over $\mathcal{P}(R)$ also minimizes the true \emph{task loss} $\mathcal{L}_\text{task}(P)$ (the quantity we ultimately care about). Throughout, we assume that every optimization problem whose argmin set is used below attains its minimum.
\begin{definition}[Task consistency]\label{def:task_consistency}
For a rate budget $R>0$, the weighted distortion $D_W$ is {task-consistent at rate $R$} if
\begin{equation}\label{eq:task_consistency}
 \argmin_{P\in\mathcal{P}(R)} D_{W}(P)\;\subseteq\;\argmin_{P\in\mathcal{P}(R)} \mathcal{L}_\text{task}(P),
\end{equation}
i.e., every minimizer of the weighted distortion over the feasible set $\mathcal{P}(R)$ also minimizes the task loss over that set. The weight map $W$ for which $D_W$ is task-consistent will be called a task-consistent weight map. 
\end{definition}

Exact consistency is often too stringent. It suffices that a compressor tuned to $D_{W}$ is \emph{near-optimal} for the task. We therefore relax Definition~\ref{def:task_consistency} via a suboptimality gap.

\begin{definition}[$\delta$-approximate task consistency]\label{def:approx_task_consistency}
For a tolerance $\delta\geq 0$, the weighted distortion $D_W$ is \emph{$\delta$-task-consistent at rate $R>0$} if every $P^\star\in\argmin_{P\in\mathcal{P}(R)}D_{W}(P)$ satisfies
\begin{equation}\label{eq:approx_task_consistency}
    \mathcal{L}_\text{task}(P^\star)\;\leq\;\min_{P\in\mathcal{P}(R)}\mathcal{L}_\text{task}(P)+\delta.
\end{equation}
Definition~\ref{def:task_consistency} is the exact case $\delta=0$.
\end{definition}

\begin{remark}
    We also extend Definitions \ref{def:task_consistency} and \ref{def:approx_task_consistency} to the case when the feasible set of channels $\mathcal{P}(R)$ is replaced by a subset $\mathcal{P}'(R) \subseteq \mathcal{P}(R)$ satisfying additional constraints.
\end{remark}
The framework suggests training a single learned compressor that takes the importance-weight vector $W$ as an input conditioning signal, so that one fixed backbone can emulate a whole family of task-specific codecs simply by changing the task abstraction $W$. At runtime, a task engine re-estimates $W$ as the active task evolves, e.g., via reinforcement learning, gradient/saliency attribution, or a user prompt, and the compressor immediately reallocates bits toward the components that $W$ marks as important, without any retraining. To obtain such a weight-conditioned codec, we do not fix the weights during training, but instead expose the model to a broad family of importance vectors, so that it learns to respond correctly to any weighting it may later be given.

\section{Theoretical Analysis: Toward Task-Consistent Weighted Distortion}\label{sec:theory}
We now analyze when the separable weighted distortion is a faithful surrogate for the task loss, and what weights this dictates. Throughout, we focus on quadratic per-component distortions $D({x}_i,{\hat{x}}_i)=||{x}_i-{\hat{x}}_i||_2^2$ for analytical clarity, so $D_{W}({X},{\hat{X}})=\sum_{i=1}^{n} w_i({X})||{x}_i-{\hat{x}}_i||_2^2$. All proofs are deferred to the appendix.

\subsection{From task loss to weighted distortion via sensitivity}

Using a first-order (local) approximation of $f$, 
\begin{equation}
    f(\hat{X})\approx f(X)+J_f(X)(\hat{X}-X), \label{approxexactee}
\end{equation}
where $J_f(X)=\nabla_{X} f\in\mathbb{R}^{m\times n}$ is the Jacobian matrix. Then
\begin{equation}
    ||f(X)-f(\hat{X})||_2^2\approx (\hat{X}-X)^T J_f(X)^T J_f(X)(\hat{X}-X). \label{approx1st}
\end{equation}
We define $G(X)\triangleq J_f(X)^T J_f(X)$. This shows that the ideal distortion in $X$-space is a quadratic form with matrix weight $G(X)$, which is generally non-separable across components. 

Let $E \triangleq \hat{X}-X$. Then, assuming scalar components, we have $E^T G(X)\,E=\sum_i g_{ii}(X)e_i^2+\sum_{i\neq j}g_{ij}(X)e_ie_j$. Taking expectations, we obtain the following first-order approximation of the task loss, valid for any fixed channel $P$:
\begin{align}
\widetilde{\mathcal{L}}_{\text{task}}(P) \triangleq \mathbb{E}[\boldsymbol{E}^T G(\boldsymbol{X})\,\boldsymbol{E}] = D_W(P) + \sum_{i \neq j}\mathbb{E}\left [ g_{ij}(\boldsymbol{X}) \boldsymbol{e}_i \boldsymbol{e}_j \right],   \label{firs3jh}
\end{align}
where $W({X}) = (g_{11}({X}), \ldots, g_{nn}({X}))$, so the separable surrogate $D_W$ differs from $\widetilde{\mathcal{L}}_\text{task}$ by exactly the off-diagonal (cross-component) term. The separable surrogate is faithful as a first-order approximation when $G(X)$ is a diagonal matrix. We call this \emph{task orthogonality} since this is a property of the task function $f$. Equivalently, $g_{ij}(X)=\langle\partial f/\partial x_i,\partial f/\partial x_j\rangle = 0$ for all $i\neq j$, meaning that the task has no first-order interaction between distinct components.

Define a regularized set of feasible test channels by
\begin{equation}\label{l3.}
\mathcal{P}_{\mathrm{reg}}(R)\triangleq \{P\in\mathcal{P}(R): \sum_{i \neq j}\mathbb{E}\left [ g_{ij}(\boldsymbol{X}) \boldsymbol{e}_i \boldsymbol{e}_j \right] = 0 \}.
\end{equation}
We refer to $\mathcal{P}_{\mathrm{reg}}(R)$ as a
\emph{cross-term-canceling channel class} for the given task, since the off-diagonal term in
\eqref{firs3jh} vanishes for every $P\in\mathcal{P}_{\mathrm{reg}}(R)$. If
the minimizations in Definitions \ref{def:task_consistency} and \ref{def:approx_task_consistency} are over the restricted set $\mathcal{P}_{\mathrm{reg}}(R)$, then the separable surrogate $D_W$ is faithful as a first-order approximation.

\subsection{Linear tasks: reduction to weighted MSE}

We first specialize to \emph{linear} tasks, for which \eqref{approxexactee} holds with
equality. Then, over
$\mathcal{P}_{\mathrm{reg}}(R)$ as defined in \eqref{l3.}, minimizing the task
loss is equivalent to minimizing a weighted MSE with closed-form
weights, as we show in the following
proposition.

\begin{prop}[Exact reduction on a cross-term-canceling channel class]
Let $f(X)=SX$ with $S\in\mathbb{R}^{m\times n}$, and let $s_{:,i}$ denote the $i$-th column of $S$. Since $f$ is linear, $J_f(X)=S$, so $g_{ij}(X) = \langle s_{:,i},s_{:,j}\rangle$ for every $X$. Then, over the
cross-term-canceling channel class $\mathcal{P}_{\mathrm{reg}}(R)$ defined in
\eqref{l3.}, minimizing the task loss is equivalent to minimizing
\begin{align}
    \mathbb{E}_P\left[
    \sum_i \|s_{:,i}\|_2^2
    (\boldsymbol{x}_i-\hat{\boldsymbol{x}}_i)^2
    \right].
    \label{msell}
\end{align}
Thus, the weighted distortion $D_W$ with weights
$w_i=\|s_{:,i}\|_2^2$ yields a task-consistent separable distortion on
$\mathcal{P}_{\mathrm{reg}}(R)$.
\label{prop:multioutput}
\end{prop}

\begin{remark}[Independent sources]
{Suppose that the components of $\boldsymbol{X}$ are independent.
Then, when minimizing \eqref{msell} over the unrestricted feasible set
$\mathcal{P}(R)$, an optimal channel can be chosen componentwise.
Moreover, such a channel can be chosen so that
$\mathbb{E}_P[\boldsymbol{e}_i\boldsymbol{e}_j]=0$ for all $i\neq j$,
and hence belongs to the cross-term-canceling channel class
$\mathcal{P}_{\mathrm{reg}}(R)$ for any linear task as defined in \eqref{l3.}.
When the source components are independent Gaussian random variables,
the componentwise channel is characterized by weighted reverse water-filling.
However, an optimal channel for \eqref{msell} need not be optimal for the
unrestricted task-loss optimization over $\mathcal{P}(R)$, since the latter
may benefit from correlated reconstruction errors.}
\end{remark}

Proposition 
\ref{prop:multioutput} establishes task consistency of the weighted distortion when the minimizations in Definition \ref{def:task_consistency} are over the cross-term-canceling channel class for the given task. We next consider an arbitrary feasible subset of channels $\mathcal{P}'(R) \subseteq \mathcal{P}(R)$. Then under a finiteness condition, we characterize the value of $\delta$ in the $\delta$-approximate task consistency framework of Definition~\ref{def:approx_task_consistency}.

\begin{prop}[$\delta$-task consistency for linear tasks]
\label{delta_prop}
Let $f({X})=S {X}$ with $S \in \mathbb{R}^{m \times n}$. Define
\begin{align}
    C_S(P)\triangleq
    \sum_{i\neq j}
    \left\langle s_{:,i},s_{:,j}\right\rangle
    \mathbb{E}_P[\boldsymbol{e}_i\boldsymbol{e}_j].
\end{align}
Let $\mathcal{P}'(R) \subseteq \mathcal{P}(R)$ be any feasible set of channels such that 
\begin{align}
    \delta_S(R)\triangleq
    \sup_{P,Q\in\mathcal{P}'(R)}
    \left|C_S(P)-C_S(Q)\right|<\infty.
\end{align}
Then, the weighted distortion $D_W$ with weights $w_i=\|s_{:,i}\|_2^2$ is $\delta_S(R)$-task-consistent at rate $R$, with the minimizations in Definition \ref{def:approx_task_consistency} taken over $\mathcal{P}'(R)$.
\end{prop}


\begin{remark}
A linear change of coordinates can recover exact separability in transformed components. In particular, let $S=U\Sigma V^T$ be a singular value decomposition of $S$. Then, the orthogonal transform $\boldsymbol{Z}=V^T\boldsymbol{X}$, $\hat{\boldsymbol{Z}}=V^T\hat{\boldsymbol{X}}$ preserves mutual information and expresses the task loss exactly as a weighted MSE with weights given by the squared singular values, independently of error correlations.
Such transformed components can, however, be dense mixtures of the original pixels or patches and therefore do not provide the per-original-component importance vector required by UniTAC. 
\end{remark}

\subsection{From local gradients to integrated gradients}

In Proposition \ref{delta_prop}, exact task consistency does not generally hold over the unrestricted feasible set of channels $\mathcal{P}(R)$. This is because the task loss is approximated by the local quadratic form in \eqref{approx1st} induced by $G(X)=J_f(X)^T J_f(X)$, which is generally non-diagonal. Retaining only the diagonal of $G$ yields the sample-dependent weights 
$w_i(X) =  g_{ii}({X})= \|\partial f/\partial x_i\|_2^2$. 
These weights make $D_{W}(P)$ $\delta$-approximately task-consistent, where the gap $\delta$ arises from the neglected cross terms
$\sum_{i\neq j}\mathbb{E}[g_{ij}(\boldsymbol{X})\boldsymbol{e}_i\boldsymbol{e}_j]$.
For nonlinear tasks, the linearization error in \eqref{approx1st} further prevents an exact reduction to a task-consistent weighted distortion.

To obtain an alternative sensitivity measure, we replace the local first-order approximation in \eqref{approx1st} with an exact line-integral representation of $f(X)-f(\hat X)$ based on path-averaged gradients. For simplicity, we present the construction for a scalar task $f:\mathbb{R}^n\to\mathbb{R}$, although the same argument applies componentwise to a vector-valued task. Applying the fundamental theorem of calculus along the straight-line path from $\hat{X}$ to $X$, we obtain  
\begin{align*}
    f(X)-f(\hat{X})&=\sum_{i=1}^n (x_i-\hat{x}_i)\,\bar{g}_i(X, \hat{X}),\\
    \bar{g}_i(X, \hat{X})&\triangleq\int_0^1 \frac{\partial f}{\partial x_i}\big(\hat{X}+\alpha(X-\hat{X})\big)\,d\alpha.
\end{align*}
The quantity $(x_i-\hat{x}_i)\bar g_i(X,\hat X)$ is precisely the integrated-gradients attribution assigned to component $i$ along the reconstruction-to-source path.
In contrast, $\partial f/\partial x_i(X)$ is a purely local sensitivity measure, depending only on $X$. As $\hat X\to X$, the path collapses and
$\bar g_i\to \partial f/\partial x_i(X)$, recovering the local gradient weights used in the first-order approximation. We write $\bar g_i$ for $\bar g_i(X,\hat X)$ once the arguments are clear from the context.

Writing $b_i\triangleq\bar{g}_i\,(x_i-\hat{x}_i)$ and applying the Cauchy--Schwarz inequality, $(\boldsymbol{1}^T B)^2\leq n\sum_i b_i^2$, then taking expectations yields the bound
\begin{equation}\label{eq:igbound}
    \mathbb{E}[(f(\boldsymbol{X})-f(\boldsymbol{\hat{X}}))^2]\;\leq\;n\,\mathbb{E}\Big[\textstyle\sum_{i=1}^n \bar{g}_i^2\,(\boldsymbol{x}_i-\boldsymbol{\hat{x}}_i)^2\Big].
\end{equation}
The right-hand side is $n$ times a separable weighted distortion with sample-dependent weights $w_i(X)=\bar g_i^2$, motivating its use as a task-aware surrogate objective for controlling the task loss. 
However, the coefficient $\bar{g}_i = \bar{g}_i(X, \hat{X})$ depends on $\hat{X}$ whereas the weights $w_i(X)$ must be
derived from the source and supplied to the codec before the
reconstruction exists. To obtain a reconstruction-independent proxy, we anchor the path at a fixed task-appropriate baseline $X_0$ (e.g., a black image), so that $\bar g_i$ becomes the integrated gradient of $f$ along the path from $X_0$ to $X$~\cite{sundararajan2017axiomatic}. This per-sample attribution map is used in our experiments (Section~\ref{sec:eval}).

\section{Codec Design: ViT-based UniTAC}\label{sec:design}

The theory of Section~\ref{sec:theory} prescribes what a task-aware codec should do (i.e., allocating rate across components in proportion to a per-component importance vector), but leaves open how a single learned compressor can realize an entire family of such allocations without retraining. We now instantiate this principle in a ViT-based codec whose token-level conditioning natively consumes $W$. 

The design has a two-stage transformer autoencoder that maps the image to a compact latent space and back, mechanisms that inject the importance map into the attention computation so that capacity is steered toward the components $W$ marks as important, and a hyperprior entropy model that turns the latent representation into a bitstream and yields a differentiable rate estimate. 
Training against the weighted distortion of \eqref{eq:separable_w_distortion} over a broad randomized family of importance maps produces a single backbone that is specialized at runtime purely by swapping $W$. Fig.~\ref{fig:architecture} shows the overall pipeline. 

\begin{figure}[t]
    \centering
    \resizebox{\linewidth}{!}{\begin{tikzpicture}[
    font=\normalsize,
    >={Latex[length=2.2mm]},
    block/.style={
        draw, rounded corners=2pt, thick,
        minimum width=3.6cm, minimum height=0.95cm,
        align=center, fill=black!3
    },
    wblock/.style={
        draw=orange!70!black, rounded corners=2pt, very thick,
        minimum width=3.6cm, minimum height=0.95cm,
        align=center, fill=orange!12
    },
    eblock/.style={
        draw=teal!55!black, rounded corners=2pt, thick,
        minimum width=3.0cm, minimum height=1.1cm,
        align=center, fill=teal!8
    },
    io/.style={align=center, font=\normalsize\itshape},
    flow/.style={->, thick},
    wflow/.style={->, thick, orange!70!black, dashed},
    eflow/.style={->, thick, teal!55!black},
]

\node[block]  (e1) {Patch Embed \\ {\footnotesize $4\times4$ patches $\rightarrow$ tokens}};
\node[wblock, below=0.6cm of e1] (e2) {Neighborhood Attn.\ $\times 4$\\{\footnotesize dilated local attention}};
\node[block,  below=0.6cm of e2] (e3) {Patch Merging\\{\footnotesize $2\times2$ tokens $\rightarrow$ token}};
\node[wblock, below=0.6cm of e3] (e4) {Sparse Global Attn.\ $\times 4$\\{\footnotesize to $T$ important tokens}};

\node[wblock, right=4.6cm of e4] (d1) {Sparse Global Attn.\ $\times 4$\\{\footnotesize to $T$ important tokens}};
\node[block,  above=0.6cm of d1] (d2) {Patch Expanding\\{\footnotesize token $\rightarrow$ $2\times2$ tokens}};
\node[block,  above=0.6cm of d2] (d3) {Neighborhood Attn.\ $\times 4$\\{\footnotesize dilated local refinement}};
\node[block,  above=0.6cm of d3] (d4) {Patch Unembed \\ {\footnotesize token $\rightarrow$ $4\times4$ patch}};

\coordinate (mid) at ($(e4)!0.5!(d1)$);
\node[eblock] (ent) at ($(mid)+(0,-1.4cm)$)
    {Hyperprior entropy model\\{\footnotesize hyper-enc.\ $\rightarrow$ hyper-latent}\\{\footnotesize hyper-dec.\ $\rightarrow$ Gaussian prior}};

\node[io, above=1.0cm of e1] (xin) {$X$};
\node[io, above=1.0cm of d4] (xout) {$\hat{X}$};

\node[draw=orange!70!black, very thick, fill=orange!12, rounded corners=1pt,
      minimum size=1.0cm, align=center, font=\footnotesize]
      (w) at ($(mid)+(0,3.8cm)$) {importance\\map $W$};

\draw[flow] (xin) -- (e1);
\draw[flow] (e1) -- (e2);
\draw[flow] (e2) -- (e3);
\draw[flow] (e3) -- (e4);
\draw[flow] (d1) -- (d2);
\draw[flow] (d2) -- (d3);
\draw[flow] (d3) -- (d4);
\draw[flow] (d4) -- (xout);

\draw[eflow] (e4.south) |- (ent.west);
\draw[eflow] (ent.east) -| (d1.south);

\node[font=\footnotesize, anchor=west] at ($(e4.south)+(0.12cm,-0.42cm)$) {$Z$};
\node[font=\footnotesize, anchor=west] at ($(d1.south)+(0.12cm,-0.42cm)$) {$\hat{Z}$};

\node[font=\footnotesize, text=black!60, anchor=east] at ($(d1.south)+(-0.12cm,-0.42cm)$) {$+$ pos.\ enc.};

\draw[wflow] (w.west) -| ([xshift=0.6cm]e2.east) -- (e2.east);
\draw[wflow] (w.west) -| ([xshift=0.6cm]e4.east) -- (e4.east);
\draw[wflow] (w.east) -| ([xshift=-0.6cm]d1.west) -- (d1.west);

\begin{scope}[on background layer]
    \node[draw=black!25, rounded corners=3pt, dashed, inner sep=0.34cm,
          fit=(e1)(e2)(e3)(e4),
          label={[black!55,font=\footnotesize\bfseries,anchor=south west]north west:Encoder}] {};
    \node[draw=black!25, rounded corners=3pt, dashed, inner sep=0.34cm,
          fit=(d1)(d2)(d3)(d4),
          label={[black!55,font=\footnotesize\bfseries,anchor=south west]north west:Decoder}] {};
\end{scope}

\end{tikzpicture}}
    \caption{End-to-end UniTAC codec. The orange path marks the importance map $W$ conditioning the encoder and decoder attention stages.}
    \label{fig:architecture}
\end{figure}
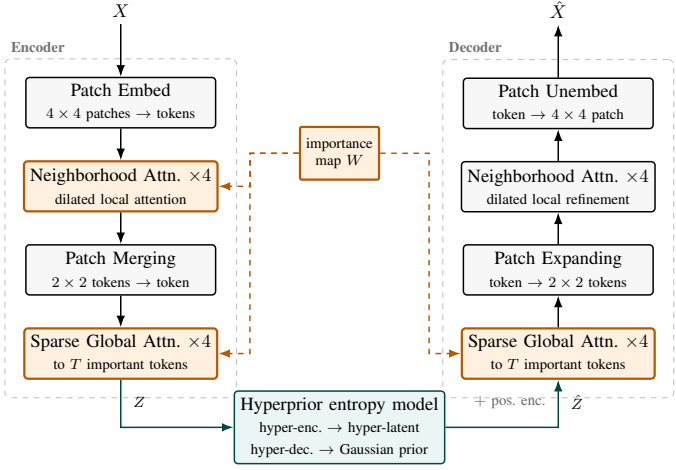

\subsection{Tokenization and transformer backbone}

The input image $X\in\mathbb{R}^{3\times H\times W}$ is split into non-overlapping $p\times p$ patches ($p=4$) by a strided 2D convolution, producing a token grid of size $\tfrac{H}{p}\times\tfrac{W}{p}$ with an embedding width of $C_0=96$. These tokens are the components $\{x_i\}$ of the framework, so the weighted distortion is measured on spatial patches, and the importance map is aligned with this grid. Every transformer block follows the standard pre-norm layout (i.e., layer normalization before each sub-layer, a residual connection around each sub-layer, and a two-layer GELU MLP whose hidden dimension is $4\times$ the token width), and all attention uses $4$ heads. Each stage below is a stack of $D=4$ such blocks.

The encoder's first stage uses \emph{neighborhood-attention} (NAT) blocks. Full self-attention costs grow quadratically with the number of tokens, making it prohibitive to apply directly on the encoder's high-resolution grid. This stage, therefore, restricts each token to attend only within a local $k\times k$ window ($k=3$), making attention linear in the number of tokens, with a learnable relative-position bias that lets the block weight neighbors by their spatial offset. Cyclically increasing dilation $\{1,2,4\}$ enlarges the receptive field as blocks stack while preserving this linear cost. A patch-merging layer then halves each spatial dimension (merging each $2\times2$ block of tokens into one) and doubles the width to $C_1=2C_0=192$. This bottleneck grid holds fewer tokens, and the second stage performs \emph{sparse global-attention} (SGAT) on it. Each SGAT block first performs sparse global attention over an importance-selected memory of tokens (Section~\ref{sec:cond}) and then a full self-attention over all tokens to restore global consistency. A linear projection maps the bottleneck tokens to a latent tensor $Z\in\mathbb{R}^{C_z\times \frac{H}{2p}\times\frac{W}{2p}}$ with $C_z=48$.

The decoder mirrors this pipeline in reverse: a stack of $D$ SGAT blocks on the bottleneck grid, a patch-expanding layer that doubles each spatial dimension, and a stack of $D$ NAT refinement blocks, ending with a sigmoid that maps the recovered tokens to a bounded reconstruction $\hat{X}\in[0,1]^{3\times H\times W}$. A 2D sinusoidal positional encoding is added at the decoder input, as its bottleneck stage would otherwise begin without any spatial reference.

\subsection{Token-level weight conditioning}\label{sec:cond}

The importance map is supplied as a low-resolution grid $W\in\mathbb{R}_+^{G\times G}$ (with $G=16$), normalized to unit mean, and bilinearly resized to each stage's token resolution. UniTAC consumes it through two complementary mechanisms that together let important tokens \emph{retain} their own detail while unimportant tokens \emph{borrow} detail from important regions, so that latent capacity concentrates where the task needs it (Fig.~\ref{fig:conditioning}).

\emph{Local stage (neighborhood attention).} Importance enters this stage through two mechanisms (Fig.~\ref{fig:conditioning}a). The first controls \emph{where} each token attends: within its $k\times k$ window, a token forms a standard attention score toward each neighbor $j$ (the scaled query--key dot product plus a learnable relative-position bias, as in base ViT), and we add an importance bonus $s\,w_j$ to it (learnable scale $s$), so every token attends more strongly toward important neighbors. The second controls \emph{how much} each token updates: a learnable self-gate $\gamma_i$ that grows with the token's own importance $w_i$ blends its current content with the attention output, $x_i\leftarrow x_i+(1-\gamma_i)\,[\mathrm{attn}(X)]_i$. An important token ($\gamma_i\!\to\!1$) keeps its own detail, while an unimportant one ($\gamma_i\!\to\!0$) fully absorbs content attended from its neighborhood. On the decoder side, the mirrored neighborhood-attention stage runs \emph{unconditioned}, without the importance bonus or self-gate: the rate has already been allocated at the encoder, so this stage only spatially refines the reconstruction.

\emph{Bottleneck stage (sparse global attention).} At the bottleneck, each query \emph{independently} samples its own memory of $T$ tokens ($T=24$), drawing token $j$ with probability $\propto w_j^{1/\tau}$, and attends only over that set (if the grid holds at most $T$ tokens, every query attends densely to all of them). The temperature $\tau$ follows a cosine schedule from $\tau=1$ (sampling proportional to $w_j$; exploratory) toward a small $\tau$ (a near-deterministic draw on the highest-importance tokens; late training and inference). A learnable gate $g_i=g_{\min}+(g_{\max}-g_{\min})\,\sigma(\beta'(w_i-1))$, with learnable slope $\beta'$ and bounds $g_{\min}\in[0,\tfrac12]$, $g_{\max}\in[\tfrac12,1]$, mixes each token's self-representation with the content read from this memory, and a following full self-attention restores global consistency. Unlike the local stage, the decoder's mirrored bottleneck (Fig.~\ref{fig:conditioning}b) \emph{retains} conditioning under $W$.

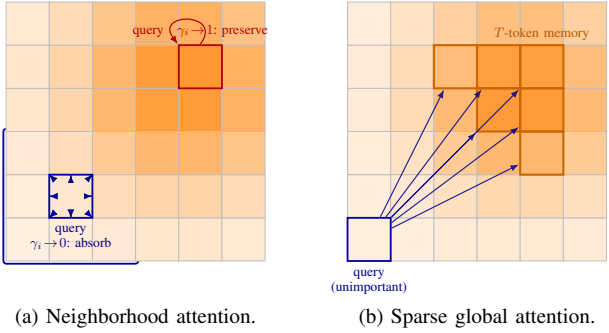
\begin{figure}[t]
    \centering
    \begin{minipage}{0.49\linewidth}
        \centering
        \resizebox{\linewidth}{!}{\begin{tikzpicture}[
    font=\normalsize,
    >={Latex[length=2.0mm]},
    cell/.style={draw=black!25, minimum size=1.0cm, inner sep=0pt},
]
\useasboundingbox (-1.3,-1.9) rectangle (6.3,6.6);

\def\cx{3.6}\def\cy{3.6}\def\sig{1.7}

\foreach \i in {0,...,5}{
  \foreach \j in {0,...,5}{
    \pgfmathsetmacro{\v}{exp(-(((\i-\cx)^2+(\j-\cy)^2)/(2*\sig*\sig)))}
    \pgfmathsetmacro{\fl}{12+72*\v}
    \node[cell, fill=orange!\fl] (n\i\j) at (\i*1.0,\j*1.0) {};
  }
}

\begin{scope}[on background layer]
  \node[draw=blue!60!black, very thick, rounded corners=2pt,
        fit=(n00)(n22), inner sep=1.6pt] {};
\end{scope}

\foreach \a/\b in {0/0,1/0,2/0,0/1,2/1,0/2,1/2,2/2}{
  \pgfmathsetmacro{\v}{exp(-(((\a-\cx)^2+(\b-\cy)^2)/(2*\sig*\sig)))}
  \pgfmathsetmacro{\tk}{0.4+2.6*\v}
  \draw[->, blue!55!black, line width=\tk pt] (n\a\b) -- (n11);
}
\node[cell, draw=blue!65!black, very thick, fill=none] at (n11) {};
\node[below=1.5pt of n11, blue!55!black, font=\small, align=center]
     {query\\ $\gamma_i\!\to\!0$: absorb};

\node[cell, draw=red!70!black, very thick, fill=none] (qa) at (n44) {};
\draw[->, red!70!black, thick]
      (qa.north) .. controls +(0.55,0.75) and +(-0.55,0.75) .. (qa.north west);
\node[above=1.5pt of qa, red!70!black, font=\small, align=center]
     {query \; $\gamma_i\!\to\!1$: preserve};

\end{tikzpicture}}\vspace{-0.4cm}\\
        {\footnotesize (a) Neighborhood attention.}
    \end{minipage}\hfill
    \begin{minipage}{0.49\linewidth}
        \centering
        \resizebox{\linewidth}{!}{\begin{tikzpicture}[
    font=\normalsize,
    >={Latex[length=2.0mm]},
    cell/.style={draw=black!25, minimum size=1.0cm, inner sep=0pt},
    mem/.style={draw=orange!80!black, line width=1.4pt},
]
\useasboundingbox (-1.3,-1.9) rectangle (6.3,6.6);

\def\cx{3.5}\def\cy{3.4}\def\sig{1.5}

\foreach \i in {0,...,5}{
  \foreach \j in {0,...,5}{
    \pgfmathsetmacro{\v}{exp(-(((\i-\cx)^2+(\j-\cy)^2)/(2*\sig*\sig)))}
    \pgfmathsetmacro{\fl}{12+72*\v}
    \node[cell, fill=orange!\fl] (n\i\j) at (\i*1.0,\j*1.0) {};
  }
}

\foreach \a/\b in {3/3,4/3,3/4,4/4,2/4,4/2}{
  \node[cell, mem, fill=none] (m\a\b) at (n\a\b) {};
}
\node[anchor=south, font=\small, orange!70!black] at (n44.north) {$T$-token memory};

\node[cell, draw=blue!65!black, very thick, fill=none] (q) at (n00) {};
\node[below=1.5pt of q, blue!55!black, font=\small, align=center]
     {query\\ (unimportant)};
\foreach \a/\b in {3/3,4/3,3/4,4/4,2/4,4/2}{
  \draw[->, blue!55!black, thick, opacity=0.85] (q) -- (m\a\b);
}

\end{tikzpicture}}\vspace{-0.4cm}\\
        {\footnotesize (b) Sparse global attention.}
    \end{minipage}
    \caption{Token-level weight conditioning on a sample token grid. Higher color intensity marks more important tokens (larger $w_i$). \textbf{(a) Local stage:} each token attends to its $k\!\times\!k$ neighborhood, biased toward important neighbors. \textbf{(b) Bottleneck stage:} each query attends to an importance-sampled memory of $T$ tokens. In both stages an importance gate lets unimportant tokens borrow detail from important ones.}
    \label{fig:conditioning}
\end{figure}

\subsection{Entropy coding and rate estimation}

The latent $Z$ is compressed with a hyperprior entropy model~\cite{balle2018variational}. A hyper-encoder maps $Z$ to a hyper-latent $Y$, coded under a factorized prior,\footnote{The hyper-encoder is a convolutional network that spatially downsamples $Z$ by $4\times$ (two stride-$2$ convolutions with LeakyReLU), and the factorized prior on $Y$ is a per-channel Gaussian with learnable mean and scale.} from which a hyper-decoder predicts per-element Gaussian parameters $(\mu,\sigma)$ for $Z$. Here, $Y$ is transmitted as side information; both encoder and decoder then derive the same $(\mu,\sigma)$ by running the hyper-decoder on it. Quantization is simulated during training by additive uniform noise and replaced by rounding at inference, with arithmetic coding producing the actual bitstream. The expected code length of the quantized latents $\hat{Z},\hat{Y}$ is estimated differentiably under the conditional model $p(\hat{Z}\mid\mu,\sigma)$ and the factorized prior $p(\hat{Y})$ as
\begin{equation}\label{eq:bpp}
    \text{bpp}=\frac{1}{HW}\Big(\textstyle\sum -\log_2 p(\hat{Z}\mid\mu,\sigma)+\sum -\log_2 p(\hat{Y})\Big),
\end{equation}
which serves as the rate term in the training objective and as the reported rate at test time. 

The importance map $W$ must also reach the decoder and is likewise transmitted as side information. As it is a small fixed-size $G\times G$ grid ($G=16$), independent of image resolution and identical for every codec we compare, its overhead is a negligible, common additive constant. We therefore omit it from the rate objective and the reported bpp.

\subsection{Training objective and randomized importance maps}

The codec is trained end to end to minimize a rate–distortion objective that pairs the rate estimate \eqref{eq:bpp} with the separable weighted distortion of \eqref{eq:separable_w_distortion}, normalized by the total weight,
\begin{equation}\label{eq:rdloss}
    \mathcal{L}=\underbrace{\frac{\sum_{i} w_i\,(x_i-\hat{x}_i)^2}{\sum_i w_i}}_{\text{weighted distortion}}+\;\lambda\,\text{bpp},
\end{equation}
where the per-patch squared error is weighted by the same importance map $W$ that conditions the network, and $\lambda$ trades off rate against task fidelity. To obtain a \emph{universal} weight-conditioned codec rather than one specialized for a single task, we do not fix $W$ during training. In fact, each image is paired with a freshly sampled importance map from a broad synthetic family (random mixtures of Gaussian ``blobs'' spanning from peaky to diffuse regimes, varied in count, location, and scale; Fig.~\ref{fig:blobs}), exposing the backbone to a rich range of importance patterns so that it responds correctly to any weighting supplied at test time. At inference, the synthetic map is replaced by a task-derived importance map, computed here via integrated gradients of a downstream classifier, so that the identical backbone is specialized to the active task. Fig.~\ref{fig:allocation} illustrates the resulting task-driven reconstructions for the same image and codec under two different importance maps.

\begin{figure}[t]
    \centering
    \includegraphics[width=0.46\linewidth]{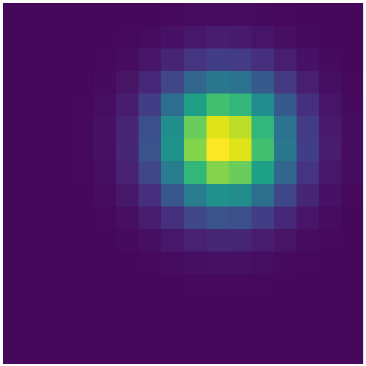}\hfill
    \includegraphics[width=0.46\linewidth]{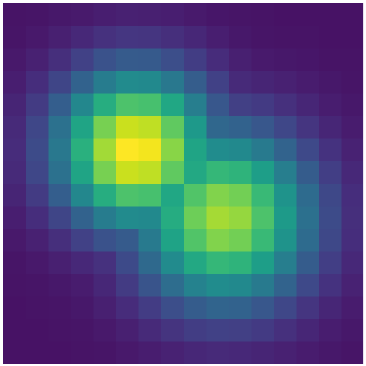}
    \caption{Synthetic training importance maps (mean-normalized; brighter denotes larger $w_i$): a single blob (left) and two overlapping blobs (right).}
    \label{fig:blobs}
\end{figure}

\begin{figure}[t]
    \centering
    \includegraphics[width=0.95\linewidth]{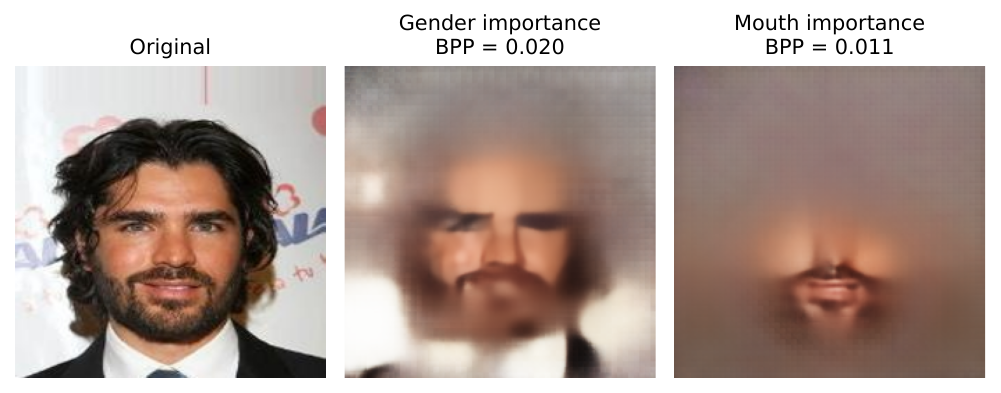}
    \caption{Task-driven bit allocation for a fixed image (single UniTAC backbone, $\lambda=0.10$). Left: original; middle/right: reconstructions conditioned on the gender and mouth integrated-gradients maps. At nearly equal rate, the same backbone shifts fidelity toward each task's important region without retraining.}
    \label{fig:allocation}
\end{figure}

\section{Experimental Evaluation}\label{sec:eval}

We evaluate whether a single weight-conditioned UniTAC backbone can, without retraining, match task-specific codecs while surpassing task-agnostic ones. On two downstream face-analysis tasks, we report signal fidelity (PSNR, semantic PSNR) and task fidelity (classification accuracy) versus rate.

\subsection{Experimental setup}\label{sec:eval_setup}

\paragraph{Data} All codecs are trained on \textbf{AffectNet}~\cite{mollahosseini2017affectnet} and evaluated on the test split of \textbf{CelebA}~\cite{liu2015faceattributes}; the CelebA training split is used only to train the classifiers that provide the downstream accuracy metric and the attribute-based importance maps. Both at training and test time, each image is cropped to a square and resized to a common size drawn from a range of scales ($64\times64$ to $256\times256$ in steps of $32$), so that the codec sees images at multiple resolutions.

\paragraph{Tasks} We consider two binary CelebA attribute classifications with different spatial support: \emph{gender} (the \texttt{Male} attribute, spatially global) and \emph{mouth state} (the \texttt{Mouth\_Slightly\_Open} attribute, spatially localized). For each task, we fine-tune a separate ImageNet-pretrained ResNet-18~\cite{he2016resnet} classifier on clean CelebA images (resized to $224\times224$ and standardized with ImageNet statistics), reaching Top-1 accuracies of $98.5\%$ (gender) and $94.3\%$ (mouth) on uncompressed images. Each classifier is used both to derive its task importance map via integrated gradients and to measure task fidelity on the reconstructions.

\paragraph{Importance maps} Each task's importance map $W$ is derived from its classifier by integrated gradients of the top-predicted-class logit, using a black (all-zero) baseline and $8$ steps along the path from the baseline to the image. The per-pixel attributions are reduced to a scalar saliency by taking the absolute value and averaging over color channels, then average-pooled to the $G\times G$ conditioning grid, floored at a small value ($0.02$), and mean-normalized. This single map serves two roles: it conditions the codec's encoder and decoder and defines the task weighting of the semantic-PSNR metric.

\paragraph{Codecs} We compare three codecs, mirroring Fig.~\ref{fig:paradigms}:
\begin{itemize}
    \item \textbf{UniTAC (ours):} the \emph{single} weight-conditioned backbone, trained once on randomized synthetic maps and specialized at test time by injecting each task's importance map.
    \item \textbf{Non-semantic (universal):} a task-agnostic codec that allocates rate evenly.
    \item \textbf{Task-based (semantic):} a separate codec trained end-to-end for each task's importance map (one for gender, one for mouth).
\end{itemize}
All three share the same ViT-based~\cite{dosovitskiy2021vit} architecture and are trained at various rate--distortion trade-offs.

\paragraph{Metrics} Against the measured bitrate (bits per pixel), we report three quantities: (i) \emph{overall PSNR}, full-image reconstruction fidelity; (ii) \emph{semantic PSNR}, reconstruction error reweighted by the task importance map, measuring fidelity where the task-relevant regions lie; and (iii) \emph{downstream Top-1 accuracy}, the task classifier evaluated on the decoded images.

\subsection{Rate--distortion: overall and semantic PSNR}\label{sec:eval_rd}

\begin{figure*}[t]
    \centering
    \includegraphics[width=\textwidth]{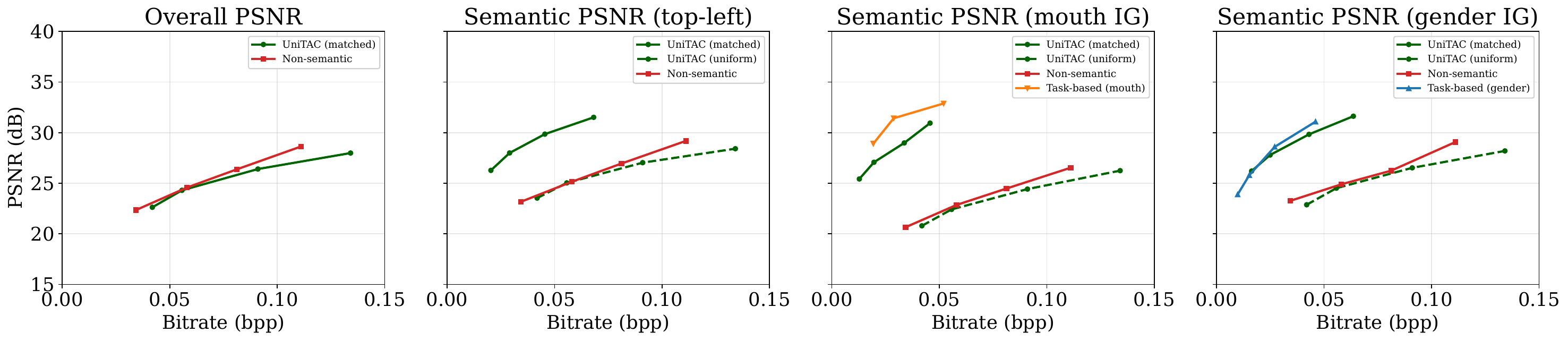}\vspace{-0.3cm}
    \caption{Rate--distortion on CelebA. Left to right: overall PSNR, and semantic PSNR under a synthetic top-left ROI, the mouth map, and the gender map.}
    \label{fig:rd}
\end{figure*}
Fig.~\ref{fig:rd} reports rate--distortion curves. On \emph{overall} PSNR (left), UniTAC under a uniform map tracks the non-semantic codec closely, confirming that weight conditioning does not sacrifice generic reconstruction quality when no task is specified. On \emph{semantic} PSNR scored under a task's importance map (middle and right), injecting the \emph{matched} map at encode time yields a large gain over both the uniform-map operating point and the non-semantic codec at equal rate: at a comparable rate ($\approx 0.04$~bpp), it improves task-region PSNR by roughly $7$~dB (gender) to $10$~dB (mouth) over uniform encoding, as rate concentrates on the task-relevant tokens. Crucially, this single UniTAC backbone nearly matches the ``task-based'' codec that is trained exclusively for that one task, with no retraining.

\subsection{Downstream task accuracy}\label{sec:eval_acc}

Fig.~\ref{fig:acc} reports downstream accuracy versus rate. For both tasks, at a comparable rate the task-matched map preserves accuracy far better than either the uniform-map operating point or the non-semantic codec. On the mouth task at $\approx\!0.034$~bpp, it retains $91.4\%$ accuracy, versus $76.9\%$ for the non-semantic codec at the same rate and $71.3\%$ under a uniform map ($0.042$~bpp). On the gender task at $\approx\!0.043$~bpp, the same backbone re-conditioned on the gender map retains $92.2\%$, versus $85.3\%$ under a uniform map ($0.042$~bpp). The smaller margin reflects the spatially global support of the gender attribute. In both cases, UniTAC's task-matched accuracy closely matches the dedicated per-task codec at equal rate with a single shared model.
\begin{figure}[t]
    \centering
    \includegraphics[width=\linewidth]{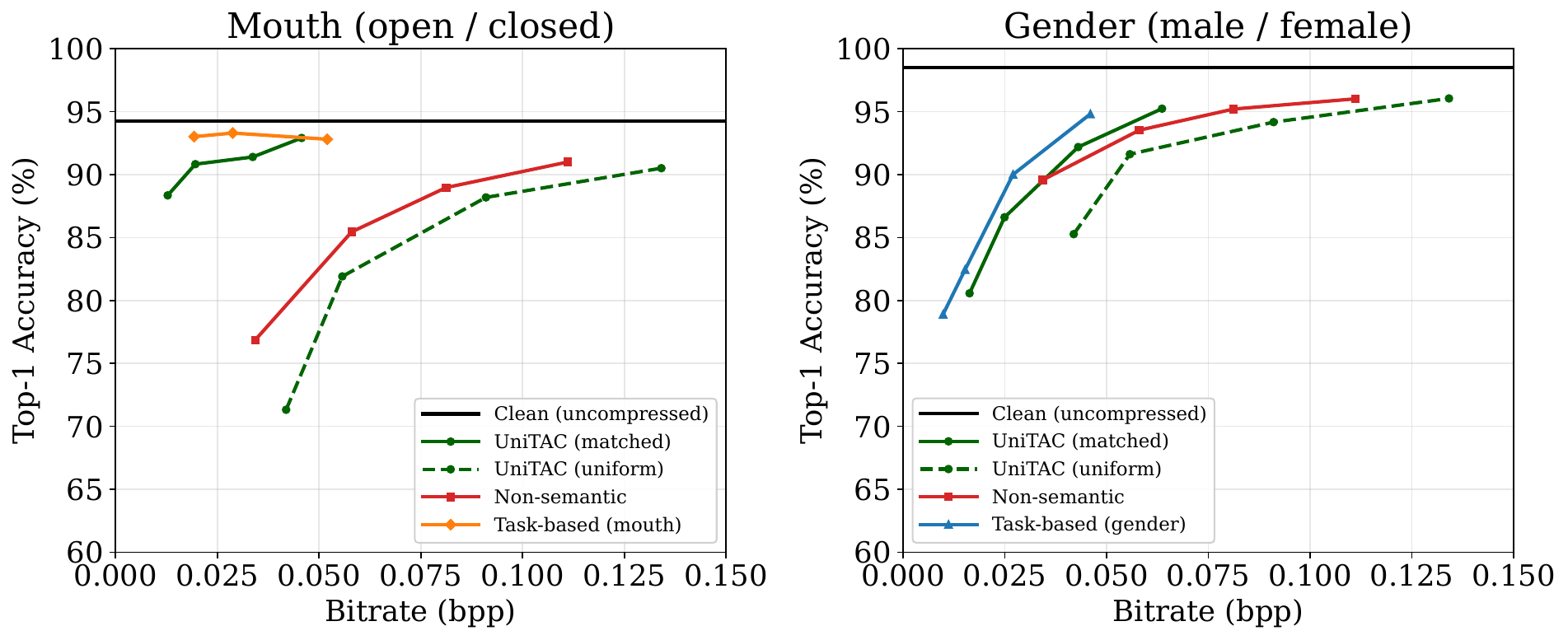}\vspace{-0.3cm}
    \caption{Downstream accuracy vs.\ rate on CelebA.}
    \label{fig:acc}
\end{figure}

\section{Additional Theoretical Properties}\label{sec:structural}
We close with two further consequences of task structure for the optimizing channels and the corresponding task-consistent weight maps. These properties are not required to instantiate the UniTAC codec of Section~\ref{sec:design}, but further validate the weighted-distortion framework of Section~\ref{sec:framework}. Under suitable conditions, we show that task \emph{symmetry} forces the weights to inherit the same symmetry, reducing to equal weights when they are sample-independent, while the task \emph{irrelevance} of a component forces its weight to zero.

We use $\pi$ to denote a permutation, i.e., a bijection from $[n]$ to $[n]$, where $[n] \triangleq \{1, \ldots, n \}$. We also write 
$\pi X \coloneqq \left(x_{\pi^{-1}(1)}, \ldots, x_{\pi^{-1}(n)}  \right)$. 
For a weight map \(W:\mathbb R^n\to\mathbb R_+^n\), define $ W^\pi(X) \triangleq
\pi\!\left(W(\pi^{-1}X)\right)$. Then $W$ is \emph{permutation-equivariant} if 
$$W^\pi(X) = W(X)$$
for every permutation $\pi$ and every \(X\in\mathbb R^n\). We say that $W$ is \(P_{\boldsymbol{X} }\)-almost surely permutation-equivariant if
\[
P_{\boldsymbol{X}}\left(
    \left\{
    X\in\mathbb R^n:
    W^\pi(X)=W(X)
    \text{ for every }\pi
    \right\}
\right)=1.
\]

Under the symmetry of the task function, we show that task-consistency of a weight map $W$ is preserved under the transformation $W\mapsto W^\pi$. We further show that, under a suitable uniqueness condition, task-consistent weight maps are permutation-equivariant. In the sample-independent case, this implies equal weights. 
 
\begin{prop}[Symmetry forces symmetric weights]\label{prop:symmetry}
    Suppose the source $\boldsymbol{X}=(\boldsymbol{x}_1,\dots,\boldsymbol{x}_n)$ is exchangeable (its components are statistically interchangeable, e.g., i.i.d.) and the task $f:\mathbb{R}^n\rightarrow\mathbb{R}^m$ is permutation-invariant, i.e., $f(\pi X)=f(X)$ for every permutation $\pi$ and every $X$ (e.g., a set function). If a weight map $W:\mathbb{R}^n\to\mathbb{R}_+^n$ is task-consistent, then the weight map $W^\pi$ is also task-consistent. If, in addition, any two task-consistent weight maps are equal $P_{\boldsymbol X}$-almost surely up to a positive scaling factor, then $W$ is $P_{\boldsymbol{X}}$-almost surely permutation-equivariant. Consequently, any sample-independent task-consistent weight map satisfies $w_1=\cdots=w_n$.
\end{prop}

The next result shows that task irrelevance constrains the information carried by an optimal channel. Specifically, an optimal channel can be reduced to the task-relevant coordinates, and when the task loss-rate function is strictly decreasing, no task-optimal channel conveys additional information about the task-irrelevant source through the task-relevant reconstruction beyond what is already available from the task-relevant source (\eqref{eq:no-nuisance-in-relevant-reconstruction}), nor carries incremental rate in the task-irrelevant reconstruction beyond what is already carried by the task-relevant reconstruction (\eqref{eq:no-incremental-ignored-rate}).

\begin{prop}[Task irrelevance eliminates task-irrelevant incremental rate]
Suppose the task ignores a subset $\mathcal{S}\subseteq[n]$, i.e.,
$f(X)=f(X')$ whenever
$x_{\mathcal{S}^c}=x'_{\mathcal{S}^c}$, and let
$\mathcal{T}\triangleq\mathcal{S}^c$.
For any channel $P=P({\hat{X}}\mid {X})$, define the reduced channel
\begin{align}
    P^{\mathrm{red}}
    (\hat x_{\mathcal T},\hat x_{\mathcal S}\mid
    x_{\mathcal T},x_{\mathcal S})
    \triangleq
    P_{\mathcal T}(\hat x_{\mathcal T}\mid x_{\mathcal T})
    \delta_c(\hat x_{\mathcal S}),
\end{align}
where $P_{\mathcal T}$ denotes the conditional law induced by
$P_{\boldsymbol X}\times P$, and $\delta_c$ denotes the point mass at any fixed reconstruction value $c$.\\
Then
\begin{align}
    \mathcal{L}_{\mathrm{task}}(P^{\mathrm{red}})
    &=
    \mathcal{L}_{\mathrm{task}}(P),
    \label{eq:reduced-same-task-loss}
    \\
    I_{P^{\mathrm{red}}}
    (\boldsymbol{X};\boldsymbol{\hat{X}})
    &\leq
    I_P
    (\boldsymbol{X};\boldsymbol{\hat{X}}).
    \label{eq:reduced-rate}
\end{align}
Consequently, at every rate $R$, there exists a task-loss minimizing
channel in $\mathcal{P}(R)$ of the reduced form above. Furthermore, 
if $L^\star(R)
    \triangleq
    \min_{P\in\mathcal{P}(R)}
    \mathcal{L}_{\mathrm{task}}(P)$
is strictly decreasing at $R$, then every task-optimal channel $P^\star
    \in
    \argmin_{P\in\mathcal{P}(R)}
    \mathcal{L}_{\mathrm{task}}(P)$
allocates all of its rate to the task-relevant coordinates, i.e., 
satisfies
\begin{align}
    I_{P^\star}
    (\boldsymbol{X};\boldsymbol{\hat{X}})
    =
    I_{P^\star}
    (\boldsymbol{X};
    \boldsymbol{\hat{x}}_{\mathcal{T}})
    =
    I_{P^\star}
    (\boldsymbol{x}_{\mathcal{T}};
    \boldsymbol{\hat{x}}_{\mathcal{T}})
    =R,
    \label{eq:all-rate-is-relevant}
\end{align}
and hence
\begin{align}
    I_{P^\star}
    (\boldsymbol{x}_{\mathcal{S}};
    \boldsymbol{\hat{x}}_{\mathcal{T}}
    \mid\boldsymbol{x}_{\mathcal{T}})
    &=0,
    \label{eq:no-nuisance-in-relevant-reconstruction}
    \\
    I_{P^\star}
    (\boldsymbol{X};
    \boldsymbol{\hat{x}}_{\mathcal{S}}
    \mid\boldsymbol{\hat{x}}_{\mathcal{T}})
    &=0.
    \label{eq:no-incremental-ignored-rate}
\end{align}
In particular, if $W$ is task-consistent at rate $R$, then every
minimizer of $D_W$ over $\mathcal{P}(R)$ satisfies
\eqref{eq:all-rate-is-relevant}--\eqref{eq:no-incremental-ignored-rate}.
\label{my_prop..}
\end{prop}



\begin{prop}[Zeroing decoupled task-irrelevant weights preserves task consistency]
Suppose the task ignores a subset
$\mathcal{S}\subseteq[n]$
and let
$\mathcal{T}:=\mathcal{S}^c$. Let $L^\star(R)$, defined in Proposition \ref{my_prop..}, be strictly decreasing. Let
$\boldsymbol{x}_{\mathcal{T}}$ and
$\boldsymbol{x}_{\mathcal{S}}$ be independent. Suppose that there exists a task-consistent weight map $W$ at rate $R$ such that $W$ is
block-separable in the sense that there exist nonnegative functions
$u_i$ and $v_i$, and
\begin{align}
    w_i(\boldsymbol{X})
    =
    \begin{cases}
        u_i(\boldsymbol{x}_{\mathcal{T}}),
        & i\in\mathcal{T},
        \\
        v_i(\boldsymbol{x}_{\mathcal{S}}),
        & i\in\mathcal{S},
    \end{cases}
    \quad
    P_{\boldsymbol{X}}\text{-almost surely}.
    \label{eq:block-separable-weights}
\end{align}
Define the weight map $W^{\mathrm{rel}}$ by
\begin{align}
    w_i^{\mathrm{rel}}(\boldsymbol{X})
    \coloneqq
    \begin{cases}
        w_i(\boldsymbol{X}),
        & i\in\mathcal{T},
        \\
        0,
        & i\in\mathcal{S}.
    \end{cases}
    \label{eq:truncated-weight-map}
\end{align}
Then, $W^{\mathrm{rel}}$ is task-consistent at rate $R$. Consequently, the existence of a block-separable task-consistent weight
map implies the existence of a task-consistent weight map that vanishes
on $\mathcal{S}$. If, in addition, any two task-consistent weight maps are equal
$P_{\boldsymbol X}$-almost surely up to a positive scaling factor, then every task-consistent weight map vanishes on $\mathcal S$ almost surely.
\label{zero_prop}
\end{prop}

\section{Discussion and Conclusion}\label{sec:conclusion}
UniTAC shows that a \emph{single} image codec can span the full range from universal to task-specialized operation, re-targeted at runtime by swapping the injected weight vector $W$ rather than retraining a codec per task. The same backbone matches the non-semantic codec on overall PSNR, substantially improves task-region semantic PSNR and downstream accuracy at equal or lower rate, and approaches the per-task ``task-based'' upper baseline. Two questions remain open. The first is the \emph{universality gap}: how closely a weight-conditioned compressor can approach a fully retrained task-specific one across a rich family of tasks. The second is the quality of the importance map, the operative link between task and rate allocation, since the codec can allocate bits only as well as the supplied weights indicate. We estimate this map by integrated gradients with a fixed baseline; comparing alternative attribution methods and richer models that capture the cross-component interactions discarded by the diagonal approximation is a promising avenue for closing the remaining gap to task-specific codecs.


\appendices
\section{}\label{app:proofs}

\subsection{Proof of Proposition \ref{prop:multioutput}}

Since $f(X) = S X$, \eqref{approxexactee} is exact and so is \eqref{approx1st}. Hence, in \eqref{firs3jh}, we have $\widetilde{\mathcal{L}}_{\text{task}}(P) = \mathcal{L}_{\text{task}}(P)$. The Jacobian matrix of $f$ is $J_f(X) = S$, so that $G(X) =  J_f(X)^T J_f(X) = S^T S$. Thus, $g_{ij}(X) = \langle s_{:,i},s_{:,j}\rangle$. Then since the second term in \eqref{firs3jh} vanishes for every channel in $\mathcal{P}_{\mathrm{reg}}(R)$, we have equality $\mathcal{L}_{\text{task}}(P) = D_W(P)$, where the weights are the diagonal components of $G$, namely $w_i(X) = \langle s_{:,i},s_{:,i}\rangle = \| s_{:,i} \|^2$.

\subsection{Proof of Proposition \ref{delta_prop}}

\begin{proof}
Let $P_W^\star \in
\argmin_{P\in {\mathcal{P}'}(R)} D_W(P)$
be any minimizer of the separable weighted distortion, and let $P_{\text{task}}^\star \in
\argmin_{P\in {\mathcal{P}'}(R)}
\mathcal{L}_{\text{task}}(P)$ be a task-loss minimizer. We have
\begin{align*}
\mathcal{L}_{\text{task}}(P)
&=
\sum_i
\|s_{:,i}\|_2^2
\mathbb{E}_P[\boldsymbol{e}_i^2]
+
\sum_{i\neq j}
\left\langle s_{:,i},s_{:,j}\right\rangle
\mathbb{E}_P[
\boldsymbol{e}_i\boldsymbol{e}_j
] \\
&=
D_W(P)+C_S(P),
\end{align*}
where $w_i=\|s_{:,i}\|_2^2$. By optimality of $P_W^\star$ for $D_W$,
\begin{align*}
D_W(P_W^\star)
\leq
D_W(P_{\text{task}}^\star).
\end{align*}
Therefore,
\begin{align*}
\mathcal{L}_{\text{task}}(P_W^\star)
&=
D_W(P_W^\star)+C_S(P_W^\star) \\
&\leq
D_W(P_{\text{task}}^\star)+C_S(P_W^\star) \\
&=
\mathcal{L}_{\text{task}}(P_{\text{task}}^\star)
+
C_S(P_W^\star)
-
C_S(P_{\text{task}}^\star) \\
&\leq
\mathcal{L}_{\text{task}}(P_{\text{task}}^\star)
+
\left|
C_S(P_W^\star)
-
C_S(P_{\text{task}}^\star)
\right| \\
&\leq
\mathcal{L}_{\text{task}}(P_{\text{task}}^\star)
+\delta_S(R),
\end{align*}
where the last inequality follows from the definition of
$\delta_S(R)$. Since $P_W^\star$ was an arbitrary minimizer of $D_W$ over
${\mathcal{P}'}(R)$, every such minimizer satisfies
\[
\mathcal{L}_{\text{task}}(P_W^\star)
\leq
\min_{P\in {\mathcal{P}'}(R)}
\mathcal{L}_{\text{task}}(P)
+\delta_S(R).
\]
{Hence, by Definition~\ref{def:approx_task_consistency} extended to $\mathcal{P}'(R)$}, the separable weighted distortion with
$w_i=\|s_{:,i}\|_2^2$ is $\delta_S(R)$-task-consistent at rate $R${, with the minimizations taken over $\mathcal{P}'(R)$.}
\end{proof}

\subsection{Proof of Proposition~\ref{prop:symmetry}}
\begin{proof}

Fix a permutation $\pi$ and a channel $P = P(\hat{X}{\mid}X)$. Let
$(\boldsymbol{X},\hat{\boldsymbol{X}})$ have joint distribution
$P_{\boldsymbol{X}}\times P$. Write $\mathbb{E}_P$ and $I_P$ for expectation
and mutual information under this joint distribution. We also define the channel $P^\pi = \pi P$, where $P^\pi(\hat{X}|X) = P(\pi^{-1}\hat{X} | \pi^{-1} X )$. By exchangeability of $\boldsymbol{X}$, the joint
distribution $P_{\boldsymbol{X}}\times(\pi P)$ is the same as the
distribution of
$(\pi\boldsymbol{X},\pi\hat{\boldsymbol{X}})$ under
$P_{\boldsymbol{X}}\times P$. Since permutations are bijections, mutual
information is invariant under the corresponding transformations. Hence,
\begin{align*}
    I_{P^\pi}
    \left(
        \boldsymbol{X};
        \hat{\boldsymbol{X}}
    \right)
    &=
    I_P
    \left(
        \pi\boldsymbol{X};
        \pi\hat{\boldsymbol{X}}
    \right) =
    I_P
    \left(
        \boldsymbol{X};
        \hat{\boldsymbol{X}}
    \right).
\end{align*}
This proves that $P \in \mathcal{P}(R) \iff \pi P \in \mathcal{P}(R)$. Using the same joint-distribution identity and the permutation invariance of the task function
$f$, we obtain $\mathcal{L}_{\text{task}}(\pi P) = \mathcal{L}_{\text{task}}(P)$. Similarly, it can be checked that $D_{W}(P) = D_{W^\pi}(\pi P)$.

The mapping $  P
    \mapsto
    \pi P$
is a bijection on the set of channels, with inverse
$P\mapsto\pi^{-1}P$. It is also a
bijection from $\mathcal{P}(R)$ onto itself. Let $\Pi_W^\star$ and $\Pi_f^\star$ denote the argmin sets in the LHS and RHS of \eqref{eq:task_consistency}, respectively. Then we have shown that 
\begin{align*}
   \Pi_f^\star &= \pi \Pi_f^\star \coloneqq  \left \{ \pi P :P \in \Pi_f^\star  \right \}
    ,
    \\
    \Pi_{W^\pi}^\star &= \pi \Pi_W^\star \coloneqq \left \{ \pi P :P \in \Pi_W^\star  \right \}.
\end{align*} 
Finally, suppose that $W$ is task-consistent. Then
$\Pi_W^\star\subseteq\Pi_f^\star$, and hence $\Pi_{W^\pi}^\star
    =
    \pi\Pi_W^\star
    \subseteq
    \pi\Pi_f^\star
    =
    \Pi_f^\star.$ Thus $W^\pi$ is also task-consistent.


Next, we prove the second part of the proposition, which additionally assumes uniqueness up to positive scaling. Consider a task-consistent weight map $W$. If $W(\boldsymbol{X}) = \boldsymbol{0}$ almost surely,
then every channel in $\mathcal{P}(R)$ would minimize $D_W(P)$, and hence every channel would also minimize $\mathcal{L}_{\text{task}}(P)$. Consequently, every weight map would be task-consistent, contradicting uniqueness up to positive scaling. Hence, a task-consistent weight map $W(\boldsymbol{X})$ cannot be almost surely equal to zero. 

Now, by the assumed uniqueness up to positive scaling, there exists a constant $c_\pi>0$ such that
\begin{align}
    W^\pi(\boldsymbol{X})
    =
    c_\pi W(\boldsymbol{X})
    \quad\text{almost surely}.
    \label{eq:wpi-proportional}
\end{align}
Recall the definition $ W^\pi(X) \coloneqq
\pi\!\left(W(\pi^{-1}X)\right).$ If we apply the permutation map $j$ times, this will be denoted as $\pi^j$, and $j$ inverse applications will be denoted as $\pi^{-j}$. Then we have $$ W^{\pi^j}(X) \coloneqq
\pi^j\!\left(W(\pi^{-j}X)\right).$$
Since $\pi$ is a permutation of the finite set $[n]$, there exists an
integer $k\geq 1$ such that $\pi^k$ is the identity permutation. Staring from \eqref{eq:wpi-proportional}, it can be shown by induction and the exchangeability assumption that 
\begin{align}
    W^{\pi^j}(\boldsymbol{X})
    =
    c_\pi^j W(\boldsymbol{X})
    \quad\text{almost surely}
    \label{eq:iterated-wpi}
\end{align}
for every $j=1,\ldots,k$. Taking $j=k$ in \eqref{eq:iterated-wpi} and using that $\pi^k$ is the
identity permutation, we obtain
\begin{align*}
    W(\boldsymbol{X})
    =
    c_\pi^k W(\boldsymbol{X})
\end{align*}
almost surely. Since $W$ is not almost surely equal to zero, this implies
$c_\pi^k=1$. Since $c_\pi>0$, we must have $c_\pi=1$. Therefore, $W^\pi(\boldsymbol{X})
    =
    W(\boldsymbol{X})$
almost surely. Since there are only finitely many permutations, this
equality holds simultaneously for every permutation. Thus, $W$ is
$P_{\boldsymbol{X}}$-almost surely permutation equivariant.

To prove the last statement in the result, suppose there exists a task-consistent weight map $U$  such that $U(X) = U \in \mathbb{R}^n_+$ for $P_{\boldsymbol{X}}$-a.e. $X$. By permutation equivariance that we just proved, we have $U = U(X) = \pi (U(\pi^{-1} X)) = \pi U $ for every $\pi$. Hence, it must be that $U = a \mathbf{1}_n$ for some constant $a > 0$. Then by the uniqueness-up-to-scaling assumption, every other task-consistent weight map $W $ must satisfy $W(X) = c_W U(X) = c_W a \mathbf{1}_n$ for some constant $c_W > 0$. \end{proof}

\subsection{Proof of Proposition \ref{my_prop..}}

Fix a channel
$P=P({\hat{X}}\mid {X})$. Let
$(\boldsymbol{X},\boldsymbol{\hat{X}})$ have joint distribution
$P_{\boldsymbol{X}}\times P$, and write $\mathbb{E}_P$ and $I_P$ for
expectation and mutual information under this joint distribution. Since
the task ignores $\mathcal{S}$, there exists a function $g$ such that $f(X)
    =
    g(x_{\mathcal{T}})$. By construction of $P^{\mathrm{red}}$, the joint distribution of
$(\boldsymbol{x}_{\mathcal{T}},
\boldsymbol{\hat{x}}_{\mathcal{T}})$ is the same under
$P_{\boldsymbol{X}}\times P^{\mathrm{red}}$ and
$P_{\boldsymbol{X}}\times P$. Therefore, $\mathcal{L}_{\mathrm{task}}(P^{\mathrm{red}})
    = \mathcal{L}_{\mathrm{task}}(P).$

Under $P_{\boldsymbol{X}}\times P^{\mathrm{red}}$,
$\boldsymbol{\hat{x}}_{\mathcal{S}}$ is constant, and
$\boldsymbol{\hat{x}}_{\mathcal{T}}$ is conditionally independent of
$\boldsymbol{x}_{\mathcal{S}}$ given
$\boldsymbol{x}_{\mathcal{T}}$. Hence,
\begin{align}
    I_{P^{\mathrm{red}}}
    (\boldsymbol{X};\boldsymbol{\hat{X}})
    &=
    I_{P^{\mathrm{red}}}
    (\boldsymbol{x}_{\mathcal{T}},
    \boldsymbol{x}_{\mathcal{S}};
    \boldsymbol{\hat{x}}_{\mathcal{T}}) \notag
    \\
    &=
    I_{P^{\mathrm{red}}}
    (\boldsymbol{x}_{\mathcal{T}};
    \boldsymbol{\hat{x}}_{\mathcal{T}}) \notag
    \\
    &=
    I_P
    (\boldsymbol{x}_{\mathcal{T}};
    \boldsymbol{\hat{x}}_{\mathcal{T}}) \notag
    \\
    &\leq
    I_P
    (\boldsymbol{X};\boldsymbol{\hat{X}}). \notag 
\end{align}
This proves \eqref{eq:reduced-same-task-loss} and
\eqref{eq:reduced-rate}. If $P$ is task-optimal in
$\mathcal{P}(R)$, then $P^{\mathrm{red}}$ is feasible and has the same
task loss, so it is also task-optimal. This proves the first part of the
proposition.

Next, suppose that $L^\star$ is strictly decreasing at $R$, and fix
\begin{align*}
    P^\star
    \in
    \argmin_{P\in\mathcal{P}(R)}
    \mathcal{L}_{\mathrm{task}}(P).
\end{align*}
Let $r
    \coloneqq
    I_{P^\star}
    (\boldsymbol{x}_{\mathcal{T}};
    \boldsymbol{\hat{x}}_{\mathcal{T}}).$
By the first part, the reduced channel
$(P^\star)^{\mathrm{red}}$ has rate $r$ and task loss
$L^\star(R)$. Suppose, for contradiction, that $r<R$. Choose $r'$ such
that $r<r'<R.$
Then $(P^\star)^{\mathrm{red}}\in\mathcal{P}(r')$, and hence
\begin{align*}
    L^\star(r')
    &\leq
    \mathcal{L}_{\mathrm{task}}
    ((P^\star)^{\mathrm{red}})
    \\
    &=
    L^\star(R).
\end{align*}
On the other hand, since $r'<R$ and $L^\star$ is strictly decreasing at
$R$, we have $  L^\star(r')
    >
    L^\star(R),$
which is a contradiction. Therefore, $I_{P^\star}
    (\boldsymbol{x}_{\mathcal{T}};
    \boldsymbol{\hat{x}}_{\mathcal{T}})
    =
    R.$
Since
\begin{align*}
    I_{P^\star}
    (\boldsymbol{x}_{\mathcal{T}};
    \boldsymbol{\hat{x}}_{\mathcal{T}})
    &\leq
    I_{P^\star}
    (\boldsymbol{X};
    \boldsymbol{\hat{x}}_{\mathcal{T}})
    \\
    &\leq
    I_{P^\star}
    (\boldsymbol{X};\boldsymbol{\hat{X}})
    \leq R,
\end{align*}
all three quantities are equal to $R$. Finally, the chain rule gives
\begin{align}
    &I_{P^\star}
    (\boldsymbol{X};\boldsymbol{\hat{X}})
    -
    I_{P^\star}
    (\boldsymbol{x}_{\mathcal{T}};
    \boldsymbol{\hat{x}}_{\mathcal{T}})
    \nonumber
    \\
    &\qquad=
    I_{P^\star}
    (\boldsymbol{x}_{\mathcal{S}};
    \boldsymbol{\hat{x}}_{\mathcal{T}}
    \mid\boldsymbol{x}_{\mathcal{T}})
    +
    I_{P^\star}
    (\boldsymbol{X};
    \boldsymbol{\hat{x}}_{\mathcal{S}}
    \mid\boldsymbol{\hat{x}}_{\mathcal{T}}).
    \label{eq:irrelevant-rate-decomposition}
\end{align}
The left-hand side of
\eqref{eq:irrelevant-rate-decomposition} is zero, and both terms on the
right-hand side are nonnegative. Therefore, both terms are zero, which
proves
\eqref{eq:all-rate-is-relevant}--\eqref{eq:no-incremental-ignored-rate}.
The last statement follows immediately from task consistency.

\subsection{Proof of Proposition \ref{zero_prop}}

Define
\begin{align*}
    D_{\mathcal{T}}(P)
    &\coloneqq
    \mathbb{E}_P
    \left[
        \sum_{i\in\mathcal{T}}
        u_i(\boldsymbol{x}_{\mathcal{T}})
        \left\|
            \boldsymbol{x}_i-\boldsymbol{\hat{x}}_i
        \right\|_2^2
    \right],
    \\
    D_{\mathcal{S}}(P)
    &\coloneqq
    \mathbb{E}_P
    \left[
        \sum_{i\in\mathcal{S}}
        v_i(\boldsymbol{x}_{\mathcal{S}})
        \left\|
            \boldsymbol{x}_i-\boldsymbol{\hat{x}}_i
        \right\|_2^2
    \right].
\end{align*}
Then under the block-separability assumption \eqref{eq:block-separable-weights}, we can write
\begin{align}
    D_W(P)
    &=
    D_{\mathcal{T}}(P)+D_{\mathcal{S}}(P), \label{givendwpn}
    \\
    D_{W^{\mathrm{rel}}}(P)
    &=
    D_{\mathcal{T}}(P). \notag 
\end{align}
We first show that any channel can be replaced by a product of its
marginal block channels without increasing its rate or changing 
the value of $D_W$ given in \eqref{givendwpn}. Fix a channel
$P=P(\hat{X}\mid X)$, and let
$P_{\mathcal{T}}
    =
    P(\hat{x}_{\mathcal{T}}
    \mid x_{\mathcal{T}})$ and $P_{\mathcal{S}}
    =
    P(\hat{x}_{\mathcal{S}}
    \mid x_{\mathcal{S}})$
denote the two conditional laws induced by
$P_{\boldsymbol{X}}\times P$. Define the product channel
$\overline{P}=P_{\mathcal{T}}P_{\mathcal{S}}$ by
\begin{align}
    \overline{P}
    ({\hat{x}}_{\mathcal{T}},
    {\hat{x}}_{\mathcal{S}}
    \mid
    {x}_{\mathcal{T}},
    {x}_{\mathcal{S}})
    \coloneqq
    P_{\mathcal{T}}
    ({\hat{x}}_{\mathcal{T}}
    \mid {x}_{\mathcal{T}})
    P_{\mathcal{S}}
    ({\hat{x}}_{\mathcal{S}}
    \mid {x}_{\mathcal{S}}).
    \label{eq:block-product-channel}
\end{align}
By the chain rule,
\begin{align*}
    I_P
    (\boldsymbol{X};\boldsymbol{\hat{X}})
    &=
    I_P
    (\boldsymbol{x}_{\mathcal{T}};
    \boldsymbol{\hat{X}})
    +
    I_P
    (\boldsymbol{x}_{\mathcal{S}};
    \boldsymbol{\hat{X}}
    \mid\boldsymbol{x}_{\mathcal{T}})
    \\
    &\geq
    I_P
    (\boldsymbol{x}_{\mathcal{T}};
    \boldsymbol{\hat{x}}_{\mathcal{T}})
    +
    I_P
    (\boldsymbol{x}_{\mathcal{S}};
    \boldsymbol{\hat{X}}
    \mid\boldsymbol{x}_{\mathcal{T}}).
\end{align*}
Since
$\boldsymbol{x}_{\mathcal{T}}$ and
$\boldsymbol{x}_{\mathcal{S}}$ are independent,
\begin{align*}
    I_P
    (\boldsymbol{x}_{\mathcal{S}};
    \boldsymbol{\hat{X}}
    \mid\boldsymbol{x}_{\mathcal{T}})
    &=
    I_P
    (\boldsymbol{x}_{\mathcal{S}};
    \boldsymbol{\hat{X}},
    \boldsymbol{x}_{\mathcal{T}})
    \\
    &\geq
    I_P
    (\boldsymbol{x}_{\mathcal{S}};
    \boldsymbol{\hat{x}}_{\mathcal{S}}).
\end{align*}
Therefore,
\begin{align}
    I_P
    (\boldsymbol{X};\boldsymbol{\hat{X}})
    \geq
    I_P
    (\boldsymbol{x}_{\mathcal{T}};
    \boldsymbol{\hat{x}}_{\mathcal{T}})
    +
    I_P
    (\boldsymbol{x}_{\mathcal{S}};
    \boldsymbol{\hat{x}}_{\mathcal{S}}).
    \label{eq:mutual-information-tensorization}
\end{align}
Under the product channel $\overline{P}$, independence of the two source
blocks gives
\begin{align}
    I_{\overline{P}}
    (\boldsymbol{X};\boldsymbol{\hat{X}})
    &=
    I_P
    (\boldsymbol{x}_{\mathcal{T}};
    \boldsymbol{\hat{x}}_{\mathcal{T}})
    +
    I_P
    (\boldsymbol{x}_{\mathcal{S}};
    \boldsymbol{\hat{x}}_{\mathcal{S}}) \notag 
    \\
    &\leq
    I_P
    (\boldsymbol{X};\boldsymbol{\hat{X}}). \label{precro}
\end{align}
Moreover, the joint distribution of
$(\boldsymbol{x}_{\mathcal{T}},
\boldsymbol{\hat{x}}_{\mathcal{T}})$ is the same under $P_{\boldsymbol{X}} \times P$ and $P_{\boldsymbol{X}} \times \overline{P}$. The same applies to the joint distribution of
$(\boldsymbol{x}_{\mathcal{S}},
\boldsymbol{\hat{x}}_{\mathcal{S}})$. Hence,
\begin{align}
\begin{split}
    D_{\mathcal{T}}(\overline{P})
    &=
    D_{\mathcal{T}}(P),
    \\
    D_{\mathcal{S}}(\overline{P})
    &=
    D_{\mathcal{S}}(P),
\end{split}\label{splitsw}
\end{align}
and thus $  D_W(\overline{P})
    =
    D_W(P).$

Now fix
\begin{align*}
    P^\star
    \in
    \argmin_{P\in\mathcal{P}(R)}
    D_W(P).
\end{align*}
Let $P_{\mathcal{T}}^\star$ and
$P_{\mathcal{S}}^\star$ denote its induced marginal block channels, and
let $\overline{P}^\star
    \coloneqq
    P_{\mathcal{T}}^\star P_{\mathcal{S}}^\star.$ Inequality \eqref{precro} shows that $\overline{P}^\star$ is feasible, and \eqref{splitsw} shows that $\overline{P}^\star$
has the same weighted distortion as $P^\star$. Therefore,
$\overline{P}^\star$ also minimizes $D_W$ over $\mathcal{P}(R)$. Since
$W$ is task-consistent, $\overline{P}^\star$ minimizes the task loss.
It follows from \eqref{eq:all-rate-is-relevant} that
\begin{align}
    I_{P_{\mathcal{T}}^\star}
    (\boldsymbol{x}_{\mathcal{T}};
    \boldsymbol{\hat{x}}_{\mathcal{T}})
    =
    R.
    \label{eq:relevant-block-uses-full-rate}
\end{align}
On the other hand, the rate of the product channel satisfies
\begin{align*}
    I_{\overline{P}^\star}
    (\boldsymbol{X};\boldsymbol{\hat{X}})
    &=
    I_{P_{\mathcal{T}}^\star}
    (\boldsymbol{x}_{\mathcal{T}};
    \boldsymbol{\hat{x}}_{\mathcal{T}})
    +
    I_{P_{\mathcal{S}}^\star}
    (\boldsymbol{x}_{\mathcal{S}};
    \boldsymbol{\hat{x}}_{\mathcal{S}})
    \\
    &\leq R.
\end{align*}
Together with \eqref{eq:relevant-block-uses-full-rate}, this implies
\begin{align}
    I_{P_{\mathcal{S}}^\star}
    (\boldsymbol{x}_{\mathcal{S}};
    \boldsymbol{\hat{x}}_{\mathcal{S}})
    =
    0.
    \label{eq:ignored-block-zero-rate}
\end{align}
We next show that $P_{\mathcal{T}}^\star$ minimizes the relevant-block
distortion $D_{\mathcal{T}}(\cdot)$ among all relevant-block channels of rate at most $R$.
Suppose, for contradiction, that there exists a channel $Q_{\mathcal{T}}
    =
    Q_{\mathcal{T}}
    (\boldsymbol{\hat{x}}_{\mathcal{T}}
    \mid\boldsymbol{x}_{\mathcal{T}})$
such that
\begin{align*}
    I_{Q_{\mathcal{T}}}
    (\boldsymbol{x}_{\mathcal{T}};
    \boldsymbol{\hat{x}}_{\mathcal{T}})
    &\leq R,
    \\
    D_{\mathcal{T}}(Q_{\mathcal{T}})
    &<
    D_{\mathcal{T}}(P_{\mathcal{T}}^\star).
\end{align*}
By \eqref{eq:ignored-block-zero-rate}, the product channel
$Q_{\mathcal{T}}P_{\mathcal{S}}^\star$ has rate at most $R$. Its
weighted distortion satisfies
\begin{align*}
    D_W
    (Q_{\mathcal{T}}P_{\mathcal{S}}^\star)
    &=
    D_{\mathcal{T}}(Q_{\mathcal{T}})
    +
    D_{\mathcal{S}}(P_{\mathcal{S}}^\star)
    \\
    &<
    D_{\mathcal{T}}(P_{\mathcal{T}}^\star)
    +
    D_{\mathcal{S}}(P_{\mathcal{S}}^\star)
    \\
    &=
    D_W(\overline{P}^\star),
\end{align*}
contradicting the optimality of $\overline{P}^\star$. Thus,
$P_{\mathcal{T}}^\star$ minimizes $D_{\mathcal{T}}$ among all
relevant-block channels of rate at most $R$.

Next, let $Q
    \in
    \argmin_{P\in\mathcal{P}(R)}
    D_{W^{\mathrm{rel}}}(P),$
and let $Q_{\mathcal{T}}
    =
    Q(\boldsymbol{\hat{x}}_{\mathcal{T}}
    \mid\boldsymbol{x}_{\mathcal{T}})$
denote the relevant-block channel induced by
$P_{\boldsymbol{X}}\times Q$. By the data-processing inequality,
\begin{align*}
    I_{Q_{\mathcal{T}}}
    (\boldsymbol{x}_{\mathcal{T}};
    \boldsymbol{\hat{x}}_{\mathcal{T}})
    \leq
    I_Q
    (\boldsymbol{X};\boldsymbol{\hat{X}})
    \leq R.
\end{align*}
Moreover,
\begin{align*}
    D_{W^{\mathrm{rel}}}(Q)
    =
    D_{\mathcal{T}}(Q_{\mathcal{T}}).
\end{align*}
The channel $Q_{\mathcal{T}}$ must minimize the relevant-block
distortion among all relevant-block channels of rate at most $R$.
Indeed, otherwise there would exist a relevant-block channel with rate
at most $R$ and strictly smaller relevant-block distortion. Extending
that channel by a constant reconstruction on $\mathcal{S}$ would produce
a channel in $\mathcal{P}(R)$ with strictly smaller
$W^{\mathrm{rel}}$-weighted distortion than $Q$, which is a
contradiction. Since both $Q_{\mathcal{T}}$ and
$P_{\mathcal{T}}^\star$ minimize the same relevant-block distortion,
\begin{align}
    D_{\mathcal{T}}(Q_{\mathcal{T}})
    =
    D_{\mathcal{T}}(P_{\mathcal{T}}^\star).
    \label{eq:equal-relevant-block-distortion}
\end{align}
Consider the product channel
$Q_{\mathcal{T}}P_{\mathcal{S}}^\star$. By
\eqref{eq:ignored-block-zero-rate}, it belongs to
$\mathcal{P}(R)$. Using
\eqref{eq:equal-relevant-block-distortion}, we obtain
\begin{align*}
    D_W
    (Q_{\mathcal{T}}P_{\mathcal{S}}^\star)
    &=
    D_{\mathcal{T}}(Q_{\mathcal{T}})
    +
    D_{\mathcal{S}}(P_{\mathcal{S}}^\star)
    \\
    &=
    D_{\mathcal{T}}(P_{\mathcal{T}}^\star)
    +
    D_{\mathcal{S}}(P_{\mathcal{S}}^\star)
    \\
    &=
    D_W(\overline{P}^\star).
\end{align*}
Thus, $Q_{\mathcal{T}}P_{\mathcal{S}}^\star$ minimizes $D_W$ over
$\mathcal{P}(R)$. Since $W$ is task-consistent,
$Q_{\mathcal{T}}P_{\mathcal{S}}^\star$ also minimizes the task loss.

The channels $Q$ and
$Q_{\mathcal{T}}P_{\mathcal{S}}^\star$ induce the same joint
distribution of
$(\boldsymbol{x}_{\mathcal{T}},
\boldsymbol{\hat{x}}_{\mathcal{T}})$. Since the task ignores
$\mathcal{S}$, there exists a function $g$ such that $f({X})
    =
    g({x}_{\mathcal{T}}).$
Consequently,
\begin{align*}
    \mathcal{L}_{\mathrm{task}}(Q)
    &=
    \mathcal{L}_{\mathrm{task}}
    (Q_{\mathcal{T}}P_{\mathcal{S}}^\star)
    \\
    &=
    L^\star(R).
\end{align*}
Since this holds for every minimizer $Q$ of
$D_{W^{\mathrm{rel}}}$ over $\mathcal{P}(R)$,
$W^{\mathrm{rel}}$ is task-consistent.

It remains to prove the assertion under uniqueness up to positive
scaling. First, note that $W^{\mathrm{rel}}$ is not almost surely equal
to zero. To see this, choose $r'<R$ and let $P'
    \in
    \argmin_{P\in\mathcal{P}(r')}
    \mathcal{L}_{\mathrm{task}}(P).$
Since $L^\star$ is strictly decreasing at $R$,
\begin{align*}
    \mathcal{L}_{\mathrm{task}}(P')
    =
    L^\star(r')
    >
    L^\star(R).
\end{align*}
The channel $P'$ belongs to $\mathcal{P}(R)$ but does not minimize the
task loss over $\mathcal{P}(R)$. The identically zero weight map cannot
therefore be task-consistent, since every channel in $\mathcal{P}(R)$
minimizes its weighted distortion. Since $W^{\mathrm{rel}}$ is
task-consistent, it is not almost surely equal to zero.

Now suppose task-consistent weight maps are unique up to positive
scaling and $P_{\boldsymbol{X}}$-almost-sure equality. Since both $W$
and $W^{\mathrm{rel}}$ are task-consistent, there exists a constant
$c>0$ such that $ W^{\mathrm{rel}}(\boldsymbol{X})
    =
    cW(\boldsymbol{X})
    $ almost surely. 
For every $i\in\mathcal{S}$, the definition of
$W^{\mathrm{rel}}$ therefore gives
\begin{align*}
    0
    =
    w_i^{\mathrm{rel}}(\boldsymbol{X})
    =
    cw_i(\boldsymbol{X})
    \quad\text{almost surely}.
\end{align*}
Since $c>0$, we have $w_i(\boldsymbol{X})
    =
    0$ almost surely 
for every $i\in\mathcal{S}$. Every other task-consistent weight map is
a positive scalar multiple of $W$, and hence also vanishes on
$\mathcal{S}$ almost surely.


\bibliographystyle{IEEEtran}
\bibliography{references}

\end{document}